\documentclass[letterpaper]{article} % DO NOT CHANGE THIS
\usepackage{aaai24}  % DO NOT CHANGE THIS
\usepackage{times}  % DO NOT CHANGE THIS
\usepackage{helvet}  % DO NOT CHANGE THIS
\usepackage{courier}  % DO NOT CHANGE THIS
\usepackage[hyphens]{url}  % DO NOT CHANGE THIS
\usepackage{graphicx} % DO NOT CHANGE THIS
\usepackage{natbib}  % DO NOT CHANGE THIS AND DO NOT ADD ANY OPTIONS TO IT
\usepackage{caption} % DO NOT CHANGE THIS AND DO NOT ADD ANY OPTIONS TO IT
\usepackage{xspace}
\usepackage{amsmath,amssymb,amsfonts}
\usepackage{amsthm}
\usepackage{enumitem}

\usepackage{multicol,multirow}
\usepackage{url}
\usepackage{xcolor}
\usepackage[ruled,linesnumbered]{algorithm2e}

\newcommand{\prob}{{Pr}}
\newcommand{\expect}{\mathbb{E}}

\newcommand\ceil[1]{\lceil#1\rceil}

\newcommand{\pX}{\mathcal{X}}
\newcommand{\pY}{\mathcal{Y}}

\newcommand{\bx}{\mathbf{x}}
\newcommand{\by}{\mathbf{y}}
\newcommand{\bb}{\mathbf{b}}

\newcommand{\bdf}{\mathbf{f}}
\newcommand{\btheta}{\boldsymbol{\theta}}

\SetKwInOut{Input}{Input}
\SetKwInOut{Output}{Output}
\SetKwInOut{Parameter}{Params}

\newtheorem{theorem}{Theorem}

\newtheorem{lemma}[theorem]{Lemma}

\newcommand{\embed}{$\mathtt{XOR}$-$\mathtt{SMC}$\xspace}
\newcommand{\xorbin}{$\mathtt{XOR}$-$\mathtt{Binary}$\xspace}

\title{Solving Satisfiability Modulo Counting for Symbolic and Statistical AI Integration \\
With Provable Guarantees}
\author{
   Jinzhao Li, Nan Jiang, Yexiang Xue
}
\affiliations{
    Department of Computer Science, Purdue University, USA\\
    \{li4255, jiang631, yexiang\}@purdue.edu
}

\begin{document}

\maketitle
\begin{abstract}
Satisfiability Modulo Counting (SMC) encompasses problems that require both  symbolic decision-making 
and statistical reasoning. 
Its general formulation captures many real-world problems at the intersection of symbolic and statistical Artificial Intelligence.
SMC searches for policy interventions to control probabilistic outcomes. 
Solving SMC is challenging because of its highly intractable nature ($\text{NP}^{\text{PP}}$-complete), incorporating statistical inference and
symbolic reasoning.
Previous research on SMC solving lacks provable guarantees and/or suffers from sub-optimal empirical performance, especially when combinatorial constraints are present. 
We propose \embed, a polynomial algorithm with access to NP-oracles, to solve highly intractable SMC problems with constant approximation guarantees. 
\embed transforms the highly intractable SMC into satisfiability problems, by replacing the model counting in SMC with SAT formulae subject to randomized XOR constraints. 
Experiments on solving  important SMC problems in AI for social good demonstrate that
\embed outperforms several baselines both in solution quality and running time. 
\end{abstract}

\section{Introduction}
\label{sec:intro}

Symbolic and statistical approaches are two fundamental driving forces of Artificial Intelligence (AI). 
Symbolic AI, exemplified by SATisfiability (SAT) and constraint programming, finds solutions satisfying constraints but requires rigid formulations and is difficult to include probabilities.
Statistical AI captures uncertainty but often lacks constraint satisfaction. 
Integrating symbolic and statistical AI remains an open field and has gained research attention recently~\cite{cpml2023,nesymas2023,neurosymbolic2023}.

Satisfiability Modulo Counting (SMC) is an umbrella problem at the intersection of symbolic and statistical AI. 
It encompasses problems that carry out symbolic decision-making (satisfiability)
\textit{mixed with} statistical reasoning (model counting). 
SMC searches for policy interventions to control probabilistic outcomes. 
Formally, SMC is an SAT problem involving predicates on model counts. 
Model counting computes the number of models (i.e., solutions) to an SAT formula. 
Its weighted form subsumes probabilistic inference on Machine Learning (ML) models.

As a motivating SMC application, stochastic connectivity optimization searches for the optimal plan to reinforce the network structure so its connectivity is preserved under stochastic events -- a central problem for a city planner who works on securing her residents multiple 
paths to emergency shelters in case of natural disasters.
This problem is useful for disaster preparation~\cite{Wu2015}, bio-diversity
protection~\cite{Dilkina2010}, internet resilience~\cite{israeli2002shortest}, social influence maximization~\cite{kempe2005influential}, energy security~\cite{Almeida2019ReducingGG}, etc. 
It requires symbolic reasoning (satisfiability) to decide which roads to 
reinforce and where to place emergency shelters, and statistical inference (model counting) to reason about the number of paths to shelters and the probabilities of natural disasters. 
Despite successes in many use cases, previous approaches \cite{Williams2005,Conrad2012,Sheldon10NetworkDesign,wu2014stochastic} found solutions \textit{lack of certifiable guarantees}, which are unfortunately in need for policy adoption in this safety-related application. 
Besides, their surrogate approximations of connectivity may overlook important probabilistic scenarios. This results in \textit{suboptimal quality} of the generated plans. 
As application domains for SMC solvers, this paper considers emergency shelter placement and supply chain network management -- two important stochastic connectivity optimization problems.  

It is challenging to solve SMC because of their highly intractable nature ($\text{NP}^{\text{PP}}$-complete) \cite{Park04MMAPComplexity} -- still intractable even with good satisfiability solvers \cite{HandbookOfSAT2009,rossi06,Braunstein2005SurveyPA} and model counters \cite{Gomes06XORCounting,Ermon13Wish,Dachlioptas2017probabilistic,Chakraborty2013scalable,KisaVCD14,cheng2012ApproxSumOp,GOGATE2012AndOr}. 
Previous research on SMC solves either a special case or domain-specific applications \cite{Dbelanger2016structured,welling2003approximate,yedidia2001generalized,Wainwright2008,Fang15GreenSecurityGame,ConitzerS06BayesianStackelberg,Sheldon10NetworkDesign}. 
The special case is called the Marginal Maximum-A-Posterior (MMAP) problem, whose decision version can be formulated as a special case of SMC \cite{Radu14AndOrMMAP,LiuI13VariationalMMAP,Maua2012AnytimeMAP,Jiang11MessagePassingMMAP,Lee2016AOBB,Wei15DecompositionBoundMMAP}.
Both cases are solved by optimizing the surrogate 
representations of the intractable model counting in variational forms \cite{liu12bBPDecision,kiselev2014policy}, or via knowledge compilation \cite{ChoiAISTATS22,Wei15DecompositionBoundMMAP,mei2018maximum} or 
 via sample average approximation \cite{kleywegt2002sample,shapiro2003monte,swamy2006stochasticOptimization,Sheldon10NetworkDesign,Dyer2006StochasticProgramming,wu2017robust,Xue15ScheduleCascade,Verweij2003SAA}. % using e.g. Stochastic Gradient Descent (SGD). 

Nevertheless, previous approaches either cannot quantify 
the quality of their solutions, or offer one-sided guarantees, 
or offer guarantees which can be arbitrarily loose.
The lack of tight guarantees results in delayed policy adoption 
in safety-related applications such as the stochastic connectivity optimization considered in this paper. 
Second, optimizing surrogate objectives without quantifying the quality of 
approximation leads to sub-optimal behavior empirically. For example, previous stochastic connectivity optimization solvers occasionally produce suboptimal plans because their surrogate approximations overlook cases of significant probability. 
This problem is amplified when combinatorial constraints are present.

%a novel roadmap  QBF-Counting problem
We propose \embed, \textbf{\textit{ a polynomial algorithm accessing NP-oracles, to solve highly intractable SMC problems with constant approximation guarantees}}. These guarantees hold with high (e.g. $> 99\%$) probability. 
The strong guarantees enable policy adoption in safety-related domains and improve the empirical performance of SMC solving (e.g., eliminating sub-optimal behavior and providing constraint satisfaction guarantees).
The constant approximation means that the solver can correctly decide
the truth of an SMC formula if tightening or relaxing the 
bounds on the model count by a multiplicative constant do not change its truth value. 
The embedding algorithms allow us to find approximate 
solutions to beyond-$\text{NP}$ SMC problems 
 via querying $\text{NP}$ oracles.
It expands the applicability of the state-of-the-art SAT solvers 
to highly intractable problems. 
%beyond NP using 

The high-level idea behind \embed~is as follows. 
Imagine a magic that randomly filters out half of the models (solutions)
to an SAT formula. 
Model counting can be approximated using this magic and an SAT solver --
we confirm the SAT formula has more than $2^k$ models if it is satisfiable 
after applying this magic $k$ times. % and 
This magic can be implemented by introducing randomized constraints. 
The idea is developed by researchers~\cite{Valiant1986uniqueSAT,jerrum1986random,Gomes2006Sampling,Gomes06XORCounting,Ermon13Wish,ermon2013embed,kuck2019adaptive,Dachlioptas2017probabilistic,Chakraborty2013scalable,Chakraborty2014DistributionAwareSA}. In these works, model counting is approximated with guarantees using polynomial algorithms accessing NP oracles. 
\embed~notices 
such polynomial algorithms can be encoded as SAT formulae.
Hence, SAT-Modulo-Counting can be written as SAT-Modulo-SAT (or equivalently SAT), when we \textit{embed} the SAT formula compiled from algorithms to solve model counting into SMC.
The constant approximation guarantee also carries. 

We evaluate the performance of \embed~on real-world stochastic connectivity optimization problems.
In particular, we consider applied problems of emergency shelter placement and supply chain management.  
For the shelter placement problem, our \embed~finds high-quality shelter assignments with less computation time and better quality than competing baselines. For wheat supply chain management, the solutions found by our \embed~are better than those found by baselines. \embed~also runs faster than baselines\footnote{The code is available at: \url{https://github.com/jil016/xor-smc}. Please refer to \url{https://arxiv.org/abs/2309.08883} for the Appendix.}.

\section{Preliminaries}
\label{sec:prelim}

\subsection{Satisfiability Modulo Theories}

Satisfiability Modulo Theory (SMT) determines the SATisfiability (SAT) of a Boolean formula, which contains predicates whose truth values are determined by the background theory. 
SMT represents a line of successful efforts to build  general-purpose logic reasoning engines, encompassing complex expressions containing bit vectors, real numbers, integers, and strings, etc ~\cite{DBLP:series/faia/BarrettSST21}. 
Over the years, many good SMT solvers are built, such as the  Z3~\cite{DBLP:conf/tacas/MouraB08,DBLP:conf/setss/BjornerMNW18} and cvc5~\cite{DBLP:conf/tacas/BarbosaBBKLMMMN22}. They play a crucial role  in   automated theorem proving, program analysis~\cite{DBLP:journals/pacmpl/FeserM0S20}, program verification~\cite{DBLP:journals/pacmpl/KSG22}, and software testing~\cite{DBLP:conf/cade/MouraB07}.

\subsection{Model Counting and Probabilistic Inference}
Model counting computes the number of models (i.e., satisfying variable assignments) to an SAT formula. 
Consider a Boolean formula $f(\mathbf{x})$, where the input $\mathbf{x}$ is a vector of Boolean variables, and the output $f$ is also Boolean. 
When we use 0 to represent false and 1 to represent true, $\sum_x f(\mathbf{x})$ computes the model count. 
Model counting is closely related to probabilistic inference and machine learning because the marginal inference on a wide range of probabilistic models can be formulated as a weighted model counting problem~\cite{chavira2008probabilistic,Xue2016MarginalMAP}. %\xyx{cite?}. 

Exact approaches for probabilistic inference and model counting are often
based on knowledge compilation \cite{Darwiche2002compilation,KisaVCD14,ChoiKisaDarwiche13,XueCD12}.
Approximate approaches include Variational methods and
sampling. Variational methods~\cite{Wainwright2008,WainwrightJW03,Sontag08tightenLP,Hazan2010NormProductBP,Flerova2011MinibucketMomentMatching} use tractable forms to approximate
a complex probability distribution.
Due to a tight relationship between counting and sampling~\cite{jerrum1986random}, sampling-based
approaches are important for model counting.
Importance sampling-based techniques such as SampleSearch~\cite{DBLP:journals/jmlr/GogateD07} is able to provide lower bounds.
Markov Chain Monte Carlo is asymptotically accurate. However, they cannot provide guarantees
except for a limited number of cases~\cite{Jerrum:1996:MCM:241938.241950,madras2002lectures}. 
The authors of~\cite{Yuille2010GaussianSampling,HazanJ12Perturbation,BalogTGW17} transform weighted integration
into optimization queries using extreme value distribution, which
today is often called the ``Gumbel trick''~\cite{DBLP:conf/iccv/PapandreouY11,DBLP:conf/iclr/JangGP17}.

\subsection{XOR Counting}\label{prelim:xor} 

There is an interesting connection between model counting and 
solving satisfiability problems subject to randomized XOR constraints. 
To illustrate this, hold $\mathbf{x}$ at $\mathbf{x}_0$, 
suppose we would like to know if $\sum_{\mathbf{y}\in \mathcal{Y}} f(\mathbf{x}_0, \mathbf{y})$ exceeds $2^q$. %of Algorithm \ref{alg:xor} 
Consider the SAT formula:
\begin{align}
    f(\mathbf{x}_0, \mathbf{y}) \wedge \mathtt{XOR}_1(\mathbf{y}) \wedge \ldots \wedge \mathtt{XOR}_q(\mathbf{y}).
    \label{eq:xor_q}
\end{align}
Here, $\mathtt{XOR}_1, \ldots, \mathtt{XOR}_q$ are randomly sampled XOR constraints. $\mathtt{XOR}_i(\mathbf{y})$ is the logical XOR or the parity of a randomly sampled subset of variables from $\mathbf{y}$. In other words, $\mathtt{XOR}_i(\mathbf{y})$ is true if and only if an odd number of these randomly sampled variables in the subset are true.

\begin{algorithm}[!t]      
 \caption{\xorbin($f$, $\mathbf{x}_0$,  $q$)}
\label{alg:xor}
        Randomly sample $\mathtt{XOR}_1(\mathbf{y}),\ldots, \mathtt{XOR}_q(\mathbf{y})$\;
        \uIf{\small $f(\mathbf{x}_0, \mathbf{y}) \wedge \mathtt{XOR}_1(\mathbf{y})\wedge \ldots \wedge \mathtt{XOR}_q(\mathbf{y})$ is satisfiable}{
            \Return{True};
        }\Else{
            \Return{False};
        }
\end{algorithm}

Formula (\ref{eq:xor_q}) is likely to be satisfiable 
if more than $2^q$ different $\mathbf{y}$ vectors render
$f(\mathbf{x}_0, \mathbf{y})$ true. 
Conversely, Formula (\ref{eq:xor_q}) is likely to be 
unsatisfiable if $f(\mathbf{x}_0, \mathbf{y})$ has 
less than $2^q$ satisfying assignments. % 
The significance of this fact is that it essentially 
transforms model counting (beyond NP) into  
 satisfiability problems (within NP). 
An intuitive explanation of why this fact holds is that 
each satisfying assignment $\mathbf{y}$ has 
50\% chance to satisfy a randomly sampled XOR constraint.
In other words, each XOR constraint ``filters out'' 
half satisfying assignments. 
For example, the number of models satisfying $f(\mathbf{x}_0, \mathbf{y}) \wedge \mathtt{XOR}_1(\mathbf{y})$ is approximately half of that satisfying $f(\mathbf{x}_0, \mathbf{y})$. 
Continuing this chain of reasoning, if $f(\mathbf{x}_0, \mathbf{y})$ has more than $2^q$ solutions, there are still satisfying assignments left after adding $q$ XOR constraints; hence formula (\ref{eq:xor_q}) is likely satisfiable. 
The reverse direction can be reasoned similarly. % 
The precise mathematical argument of 
the constant approximation is in Lemma \ref{th:wish}. %of Algorithm \ref{alg:xor} 
%
%
%This results in . 
\begin{lemma}\cite{jerrum1986random,Gomes06XORCounting,Ermon13Wish}
\label{th:wish}
Given Boolean function $f(\mathbf{x}_0, y)$ as defined above,
\begin{itemize}[align=left, leftmargin=0pt, labelwidth=0pt, itemindent=!]
        \item If $\sum_{\mathbf{y}} f(\mathbf{x}_0, \mathbf{y}) \geq 2^{q_0}$, then for any $q \leq q_0$, with probability $1-\frac{2^c}{(2^c-1)^2}$, \xorbin$(f, \mathbf{x}_0, q-c)$ returns True.
        \item If $\sum_{\mathbf{y}} f(\mathbf{x}_0, \mathbf{y}) \leq 2^{q_0}$, then for any $q \geq q_0$, with probability $1-\frac{2^c}{(2^c-1)^2}$, \xorbin$(w, \theta_0, q+c)$ returns False.
\end{itemize}
\end{lemma}
This idea of transforming model counting problems into SAT problems subject
to randomized constraints  
is rooted in Leslie Valiant's seminal work on unique SAT \cite{Valiant1986uniqueSAT,jerrum1986random} and has been developed by a rich line of work \cite{Gomes2006Sampling,Gomes06XORCounting,Ermon13Wish,ermon2013embed,kuck2019adaptive,Dachlioptas2017probabilistic, Chakraborty2013scalable,Chakraborty2014DistributionAwareSA}. 
This idea has recently gathered momentum thanks to the rapid progress in SAT solving~\cite{ManevaMW07,Braunstein2005SurveyPA}. 
The contribution of this work extends the success of SAT solvers to problems with even higher complexity, namely, $\text{NP}^{\text{PP}}$-complete SMC problems.

\section{Problem Formulation}

Satisfiability Modulo Counting (SMC) is Satisfiability Modulo Theory (SMT) \cite{BSST09} 
with model counting as the background theory. 
A canonical definition of the SMC problem is to determine if there exists $\mathbf{x}=(x_1, \dots, x_n)\in\{0,1\}^n$ and $\mathbf{b}=(b_1, \dots, b_k)\in \{0,1\}^k$ that satisfies the formula:
\begin{equation} \small
    \phi(\mathbf{x}, \mathbf{b}), 
     b_i \Leftrightarrow \left(\sum_{\mathbf{y}_i\in\mathcal{Y}_{i}} f_i(\mathbf{x}, \mathbf{y}_i) \geq 2^{q_i}\right), \forall i \in \{1.., k\}.
    \label{eq:greatoverall}
\end{equation}
Here each $b_i$ is a Boolean predicate that is true if and only if the corresponding model count exceeds a threshold. 
Bold symbols (i.e., $\mathbf{x}$, $\mathbf{y}_i$ and $\mathbf{b}$) are vectors of Boolean variables. 
$\phi, f_1, \dots, f_k$ are Boolean functions (i.e., their input is Boolean vectors, and their outputs are also Boolean). 
We use $0$ to represent false and $1$ to represent true. Hence $\sum f_i$ computes the number
of satisfying assignments (model counts) of $f_i$. 
The directions of the inequalities do not matter much because one can always negate each $f_i$.
For instance, let $f(\bx,\by)$ be a Boolean function (output is 0 or 1). $\sum_{\by\in \pY} f(\bx,\by) \leq 2^q$ can be converted by negating $f$ and modifying the threshold to $|\pY|-2^q$, resulting in an equivalent predicate $\sum_{\by\in \pY} (\neg f(\bx,\by)) \geq |\pY| - 2^q$.

Our \embed~algorithm obtains the constant approximation guarantee to the following slightly relaxed SMC 
problems. 
 The problem $\mathtt{SMC}(\phi, f_1, \dots, f_k,$ $q_1, \dots, q_k)$ finds a satisfying assignment $(\mathbf{x}, \mathbf{b})$ for:
\begin{align}
    \phi(\mathbf{x}, \mathbf{b}) &\wedge
     \left[b_1 \Rightarrow \left(\sum_{\mathbf{y}_1\in\mathcal{Y}_{1}} f_1(\mathbf{x}, \mathbf{y}_1) \geq 2^{q_1}\right) \right] \nonumber \\
     &\dots \wedge \left[b_k \Rightarrow \left(\sum_{\mathbf{y}_k\in\mathcal{Y}_{k}} f_k(\mathbf{x}, \mathbf{y}_k) \geq 2^{q_k}\right) \right].
    \label{eq:overall}
\end{align}
The only difference compared to the full-scale problem in Eq.~\eqref{eq:greatoverall})
is the replacement of $\Leftrightarrow$ with $\Rightarrow$. This change allows us to derive 
a concise constant approximation bound. 
We also mention that all the applied SMC problems considered in this paper can be formulated in this relaxed form. 
We thank the reviewers for pointing out 
{the work of \cite{fredrikson2014satisfiability}, 
who came up with a slightly different
SMC formulation with focused applications in privacy and an exact solver. 
Their formulation was a little more
general than ours, since theirs allows for
predicates like $\sum f \geq \sum g$, while ours only allows for $\sum f \geq$ constant. 
However, our formulation can handle $\sum f \geq \sum g$ by formulating it with $(\sum f \geq \alpha) \wedge (\sum g \leq \alpha)$ and binary searching on $\alpha$. 
%solves SMC problems exactly. However, they struggle with scalability issues when confronted with millions of variables and constraints, a key challenge that our work primarily addresses.
}
\section{The \embed~Algorithm
}

The key motivation behind our proposed \embed~algorithm is to notice that Algorithm~\ref{alg:xor} itself can be written as a 
Boolean formula due to the Cook-Levin reduction.
When we embed this Boolean formula into Eq.~\eqref{eq:overall}, 
the Satisfiability-Modulo-Counting problem translates into a Satisfiability-Modulo-SAT problem, or equivalently, an SAT problem. 
This embedding also ensures a constant approximation guarantee (see Theorem \ref{th:xor-smc}).

To illustrate the high-level idea, let us consider replacing each $\sum_{\mathbf{y}_i\in\mathcal{Y}_{i}} f_i(\mathbf{x}, \mathbf{y}_i) \geq 2^{q_i}$ in Eq. (\ref{eq:overall}) with formula 
\begin{align}
    %\mathtt{XOR\_Binary}(w, a_0, k):~~~
    f_i(\mathbf{x}, \mathbf{y}_i) \wedge \mathtt{XOR}_1(\mathbf{y}_i) \wedge \ldots \wedge \mathtt{XOR}_{q_i}(\mathbf{y}_i).
    \label{eq:xor_qi}
\end{align}
We denote the previous equation (\ref{eq:xor_qi}) as $\gamma(f_i, \mathbf{x}, q_i, \mathbf{y}_i)$. This replacement results in the Boolean formula:
\begin{align}
    \phi(\mathbf{x}, \mathbf{b})\wedge 
    &\left[b_1 \Rightarrow \gamma(f_1, \mathbf{x}, q_1, \mathbf{y}_1) \right]\wedge \dots \wedge\nonumber \\
    & \left[b_k \Rightarrow \gamma(f_k, \mathbf{x}, q_k, \mathbf{y}_k) \right].
    \label{eq:embedded}
\end{align}

\begin{figure}[!t]
    \centering
    \includegraphics[width=\linewidth]{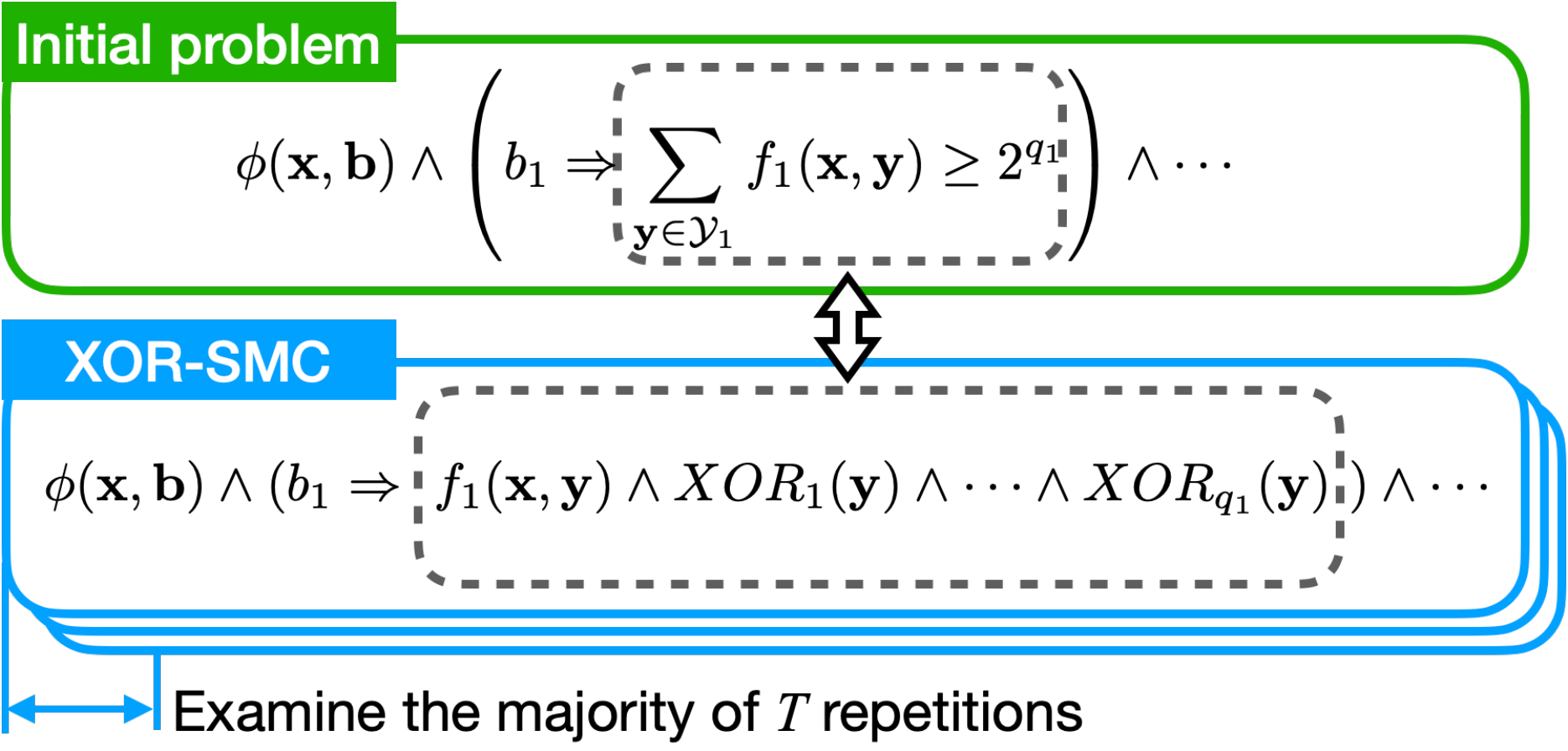}
    \caption{Our \embed~(shown in Algorithm \ref{alg:xor-smc}) solves the intractable model counting with satisfiability problems subject to randomized XOR constraints and obtains constant approximation guarantees for SMC.}
    \label{fig:xorsmc}
\end{figure}

We argue that the satisfiability of formula~\eqref{eq:embedded} should be closely 
related to that of formula~\eqref{eq:overall} due to the connection between model counting and satisfiability testing subject to randomized constraints (discussed in Section \ref{prelim:xor}).
To see this, Eq.~\eqref{eq:embedded} is satisfiable if and only if there exists $(\mathbf{x}, \mathbf{b}, \mathbf{y}_1, \ldots, \mathbf{y}_k)$ that render Eq.~\eqref{eq:embedded} true (notice $\mathbf{y}_1, \ldots, \mathbf{y}_k$ are also its variables). 
Suppose $\texttt{SMC}(\phi, f_1, \dots, f_k, q_1+c, \dots, q_k+c)$ is satisfiable (a.k.a., Eq.~\eqref{eq:overall} is satisfiable when $q_i$ is replaced with $q_i + c$).
Let $(\mathbf{x}, \mathbf{b})$ be a satisfying assignment. 
For any $b_i=1$ (true) in $\mathbf{b}$, we must have  $\sum_{\mathbf{y_{i}}\in\mathcal{Y}_{i}} f_{i}(\mathbf{x}, \mathbf{y}_{{i}}) \geq 2^{q_{i}+c}$. This 
implies with a good chance, there exists a $\by_i$ that renders $\gamma(f_i, \mathbf{x}, q_i, \mathbf{y}_i)$  true.  This is due to the discussed connection between model counting and SAT solving 
subject to randomized constraints. 
Hence $b_i \Rightarrow \gamma(f_i, \mathbf{x}, q_i, \mathbf{y}_i)$ is true. 
For any $b_i=0$ (false), the previous equation is true by default. 
Combining these two facts and $\phi(\mathbf{x}, \mathbf{b})$ is true, we see Eq.~\eqref{eq:embedded} is true.

% %
% Without losing generality, suppose $b_{i_1}, \ldots, b_{i_s}$ is true in the satisfying assignment. This implies 
% %\begin{equation*}

% %\end{equation*}
% holds for all $j=1, \ldots, s$. 
% %
% As a result, Algorithm  for all $j=1, \ldots, s$ with high probability. %This implies
% As a consequence, $(\mathbf{x}, \mathbf{b})$ is a satisfying assignment of formula~\eqref{eq:embedded})  with high probability. 

Conversely, suppose $\texttt{SMC}(\phi, f_1, \dots, f_k, q_1-c, \dots, q_k-c)$ is not satisfiable.
This implies for every $(\mathbf{x}, \mathbf{b})$, either $\phi(\mathbf{x}, \mathbf{b})$ is false, or there exists at least one $j$ such that $b_j$ is true, but $\sum_{\mathbf{y}_{{j}}\in\mathcal{Y}_{{j}}} f_{j}(\mathbf{x}, \mathbf{y}_{{j}}) < 2^{q_{j}-c}$. 
The first case implies Eq.~\eqref{eq:embedded} is false under the assignment. 
For the second case, $\sum_{\mathbf{y}_{{j}}\in\mathcal{Y}_{{j}}} f_{j}(\mathbf{x}, \mathbf{y}_{{j}}) < 2^{q_{j}-c}$ implies with a good chance there is no $\mathbf{y}_j$ to make $\gamma(f_j, \mathbf{x},  q_{j}, \mathbf{y}_j)$ true. 
Combining these two facts, with a good chance Eq.~\eqref{eq:embedded} is not satisfiable.

In practice, to reduce the error probability the determination of the model count needs to rely on the majority satisfiability status of a series of equations \eqref{eq:xor_qi} (instead of a single one). 
Hence we develop Algorithm \ref{alg:xor-smc}, which is a little bit more complex than the high-level idea discussed above.
The idea is still to  \textbf{\textit{transform the highly intractable 
SMC problem into solving an SAT problem of its polynomial size}}, while \textbf{\textit{ensuring
 a constant approximation guarantee}}. 
Fig.~\ref{fig:xorsmc} displays the encoding of Algorithm~\ref{alg:xor-smc}.
We can see the core is still to replace the intractable model counting with satisfiability problems subject to randomized constraints. 
We prove \embed has a constant approximation guarantee in Theorem~\ref{th:xor-smc}.  We leave the implementation of \embed in the Appendix~B.
% We can show:

% \proofsketch 

\begin{algorithm}[t]
\caption{\embed($\phi$, $\{f_i\}_{i=1}^k$, $\{q_i\}_{i=1}^k$, $\eta$, $c$)}
\label{alg:xor-smc}
$T \gets \ceil{\frac{(n+k)\ln2 - \ln \eta}{\alpha(c,k)}}$\;

\For{$t=1 \text{ to } T$}{
\For{$i=1\text{ to } k$}{
    $\psi_{i}^{(t)} \gets f_i(\bx, \by_i^{(t)})$\;
\For{$j=1,\ldots,q_i$}{
    $\psi_{i}^{(t)} \gets \psi_{i}^{(t)} \wedge \mathtt{XOR}_j(\by_i^{(t)})$\;
}
$\psi_{i}^{(t)} \gets \psi_{i}^{(t)} \vee \neg b_i$\;
}
$\psi_{t} \gets \psi_{1}^{(t)} \wedge \dots \wedge \psi_{k}^{(t)}$\;
}
$\phi^* \leftarrow \phi \wedge  \mathtt{Majority}(\psi_1, \ldots, \psi_T)$ \;
\eIf{there exists $(\bx, \bb, \{\by_i^{(1)}\}_{i=1}^k, \ldots, \{\by_i^{(T)}\}_{i=1}^k )$ that satisfies $\phi^*$}{
    \Return{True};
}{
    \Return{False};
}
\end{algorithm}

\begin{theorem}
\label{th:xor-smc}
Let $0 < \eta < 1$ and $c \geq \log(k+1) + 1$. Select $T = \ceil{ {((n+k)\ln2 - \ln \eta)}/{\alpha(c,k)}}$, we have 
%$\frac{\ln|\Theta| -  \ln \eta}{\alpha^*(c-1)}$,
\begin{itemize}[align=left, leftmargin=0pt, labelwidth=0pt, itemindent=!]
    \item Suppose there exists $\bx_0 \in \{0,1\}^n$ and $\bb_0 \in \{0,1\}^k$, such that $\mathtt{SMC}(\phi, f_1, \dots, f_k, q_1+c, \dots, q_k+c)$ is true. In other words, 
    \begin{equation*}
    \phi(\bx_0,\bb_0) \wedge \left( \bigwedge_{i=1}^k \left(b_i \Rightarrow \sum_{\by_i} f_i(\bx_0, \by_i) \geq 2^{q_i +c} \right) \right),
    \end{equation*}
    Then algorithm \embed($\phi$, $\{f_i\}_{i=1}^k$, $\{q_i\}_{i=1}^k$, $T$) returns true with probability greater than $1-\eta$.
    \item Contrarily, suppose $\mathtt{SMC}(\phi, f_1, \dots, f_k, q_1-c, \dots, q_k-c)$ is not satisfiable. In other words, for all $\bx \in \{0,1\}^n$ and $\bb \in \{0,1\}^k$,
    \begin{align*}
        \neg \left( \phi(\bx,\bb) \wedge \left( \bigwedge_{i=1}^k \left(b_i \Rightarrow \sum_{\by_i} f_i(\bx, \by_i) \geq 2^{q_i - c} \right) \right) \right),
    \end{align*}   
    then \embed($\phi$, $\{f_i\}_{i=1}^k$, $\{q_i\}_{i=1}^k$, $T$) returns false with probability greater than $1-\eta$. 
\end{itemize}
\end{theorem}

\begin{proof}
{\bf Claim 1:} Suppose there exists $\bx_0 = [x_1,\dots,x_n] \in \{0,1\}^n$ and $\bb_0 = [b_1,\dots,b_k] \in \{0,1\}^k$, such that 
    \begin{align} \label{eq:condition1}
        \phi(\bx_0,\bb_0) \wedge \left( \bigwedge_{i=1}^k \left(b_i \Rightarrow \sum_{\by_i} f_i(\bx_0, \by_i) \geq 2^{q_i +c} \right) \right)
    \end{align}
    holds true.
    Denote $k_0$ as the number of non-zero bits in $\bb_0$. Without losing generality, suppose those non-zero bits are the first $k_0$ bits, i.e., $b_1 = b_2 = \dots = b_{k_0} = 1$ and $b_{i} = 0, \forall i > k_0$. Then Eq.~\eqref{eq:condition1} can be simplified to:
    \begin{align} \label{eq:part_condition1}
        \phi(\bx_0,\bb_0) \wedge \left( \bigwedge_{i=1}^{k_0} \left(\sum_{\by_i} f_i(\bx_0, \by_i) \geq 2^{q_i +c} \right) \right)
    \end{align}

     Consider the Boolean formula $\psi_t$ defined in the \embed algorithm (choosing any $t \in \{1, \dots, T\}$). $\psi_t$ can be simplified by substituting the values of $\bx_0$ and $\bb_0$. After simplification, we obtain:
     \begin{equation*}
     \begin{aligned}
        \psi_t &=  \left( f_1(\bx_0, \by_1^{(t)}) \wedge \mathtt{XOR}_1(\by_1^{(t)})  \dots \wedge \mathtt{XOR}_{q_1}(\by_1^{(t)}) \right) \wedge \dots  \\
        & \wedge \left( f_{k_0}(\bx_0, \by_{k_0}^{(t)}) \wedge \mathtt{XOR}_1(\by_{k_0}^{(t)})  \dots \wedge \mathtt{XOR}_{q_{k_0}}(\by_{k_0}^{(t)}) \right).
    \end{aligned}    
     \end{equation*}
    Let $\gamma_i = \left( f_i(\bx_0, \by_i^{(t)}) \wedge \mathtt{XOR}_1(\by_i^{(t)})  \dots \wedge \mathtt{XOR}_{q_i}(\by_i^{(t)}) \right) $. Observing that $\sum_{\by_i} f_i(\bx_0, \by_i) \geq 2^{q_i +c}, \forall i = 1,\dots,k_0$. According to Lemma~\ref{th:wish}, with probability at least $1 - \frac{2^c}{(2^c - 1)^2}$, there exists $\by_i^{(t)}$, such that $(\mathbf{x}_0, \by_i^{(t)})$ renders $\gamma_i$ true. 
    The probability that $\psi_t$ is true under $(\mathbf{x}_0, \mathbf{b}_0, \by_1^{(t)}, \dots, \by_k^{(t)})$ is: 
    \begin{align*}
     \prob ( (\mathbf{x}_0, &\mathbf{b}_0, \by_1^{(t)}, \dots, \by_k^{(t)}) \text{ renders } \psi_t \text{ true})\\
     =~& \prob \left(\bigwedge_{i = 1}^{k_0} ((\mathbf{x}_0, \by_i^{(t)}) \text{ renders }\gamma_i  \text{ false}) \right) \\
    =~& 1 - \prob \left(\bigvee_{i = 1}^{k_0} ((\mathbf{x}_0, \by_i^{(t)}) \text{ renders }\gamma_i  \text{ false}) \right) \\
        \geq~& 1 - \sum_{i=1}^{k_0} \prob\left((\mathbf{x}_0, \by_i^{(t)}) \text{ renders }\gamma_i  \text{ false}\right) \\
        \geq~& 1 - \frac{k_0 2^c}{(2^c - 1)^2} \geq 1 - \frac{k 2^c}{(2^c - 1)^2}.
    \end{align*}

Define $\Gamma_t$ as a binary indicator variable where
\begin{align*}
    \Gamma_t = \begin{cases}
        1 & \text{if $(\mathbf{x}_0, \mathbf{b}_0, \by_1^{(t)}, \dots, \by_k^{(t)})$ renders $\psi_t$ true}, \\
        0 & \text{otherwise}.
    \end{cases}
\end{align*}
Therefore $\prob(\Gamma_t = 0) \leq \frac{k 2^c}{(2^c - 1)^2}$. $\prob(\Gamma_t = 0) < \frac{1}{2}$ when $c \geq \log_2(k+1) +1$. 
\embed returns true if the majority of $\psi_t, t=1,\dots,T$ are true; that is, $\sum_t \Gamma_t \geq \frac{T}{2}$. Let's define 
\begin{align*}
    \alpha(c,k) &= D\left(\frac{1}{2}\|\frac{k2^c}{(2^c - 1)^2}\right) \\
    &= \frac{1}{2} \ln \frac{(2^c - 1)^2}{k 2^{c+1}} + \left( 1-\frac{1}{2}\right) \ln \frac{2(2^c - 1)^2}{(2^c - 1)^2 - k 2^{c+1}}.
\end{align*}

When $c \geq \log_2(k+1) +1$, observing that $\alpha(c, k) > 0$ , we can apply the Chernoff-Hoeffding theorem to obtain:
\begin{align*}
    \prob\left( \sum_{t = 1}^T \Gamma_t \geq \frac{T}{2}  \right)
    &=1 - \prob\left( \sum_{t = 1}^T \Gamma_t < \frac{T}{2}  \right) \geq 1 - e^{-\alpha(c,k) T}
\end{align*}

For $T \geq \ceil{ \frac{((n+k)\ln2 - \ln \eta)}{\alpha(c,k)}} \geq\frac{- \ln \eta}{ \alpha(c,k)}$, it follows that $e^{-\alpha(c,k) T} \leq \eta$. Therefore, with a probability at least $1-\eta$, we have $\sum_t \Gamma_t \geq \frac{T}{2}$. In this scenario, \embed($\phi$, $\{f_i\}_{i=1}^k$, $\{q_i\}_{i=1}^k$, $T$) returns true as it discovers $\bx_0$, $\bb_0$, $(\by_1^{(t)}, \dots, \by_k^{(t)})$, for which the majority of Boolean formulae in $\{\psi_t\}_{t = 1}^T$ are true.

{\bf Claim 2:} Suppose for all $\bx \in \{0,1\}^n$ and $\bb \in \{0,1\}^k$, 
\begin{align*}
    \neg \left( \phi(\bx,\bb) \wedge \left( \bigwedge_{i=1}^k \left(b_i \Rightarrow \sum_{\by_i} f_i(\bx, \by_i) \geq 2^{q_i - c} \right) \right) \right) %\\
\end{align*}
Consider a fixed $\bx_1$ and $\bb_1$,
the previous condition with high probability renders most $\psi_t$ false
in Algorithm \ref{alg:xor-smc}. 
We prove that the probability is sufficiently low such that \embed will return false with a high probability after examining all $\bx$ and $\bb$. The detailed proof is left in Appendix~A.
\end{proof}

\section{Experiment 1: Locate Emergency Shelters}
% In this section, we show our \embed~finds better shelter location assignments (more emergency vacate paths) than baseline approaches  (in Fig.~\ref{fig:shelter_line}). It also needs less computational time  (in Table~\ref{tab:shelter}).

\noindent\textbf{Problem Formulation.}
Disasters such as hurricanes and floods continue to endanger millions of lives. Shelters are safe zones that protect residents from possible damage, and evacuation routes are the paths from resident zones toward shelter areas. To enable the timely evacuation of resident zones, picking a set of \textit{shelter locations} with sufficient routing from resident areas should be considered.
Given the unpredictability of chaos during natural disasters, it is crucial to guarantee multiple paths rather than one path from residential areas to shelters. This ensures that even if one route is obstructed, residents have alternative paths to safety areas.%, emphasizing the necessity of planning beyond mere reachability. 

\begin{figure}[!t]
    \centering
\includegraphics[width=0.92\linewidth]{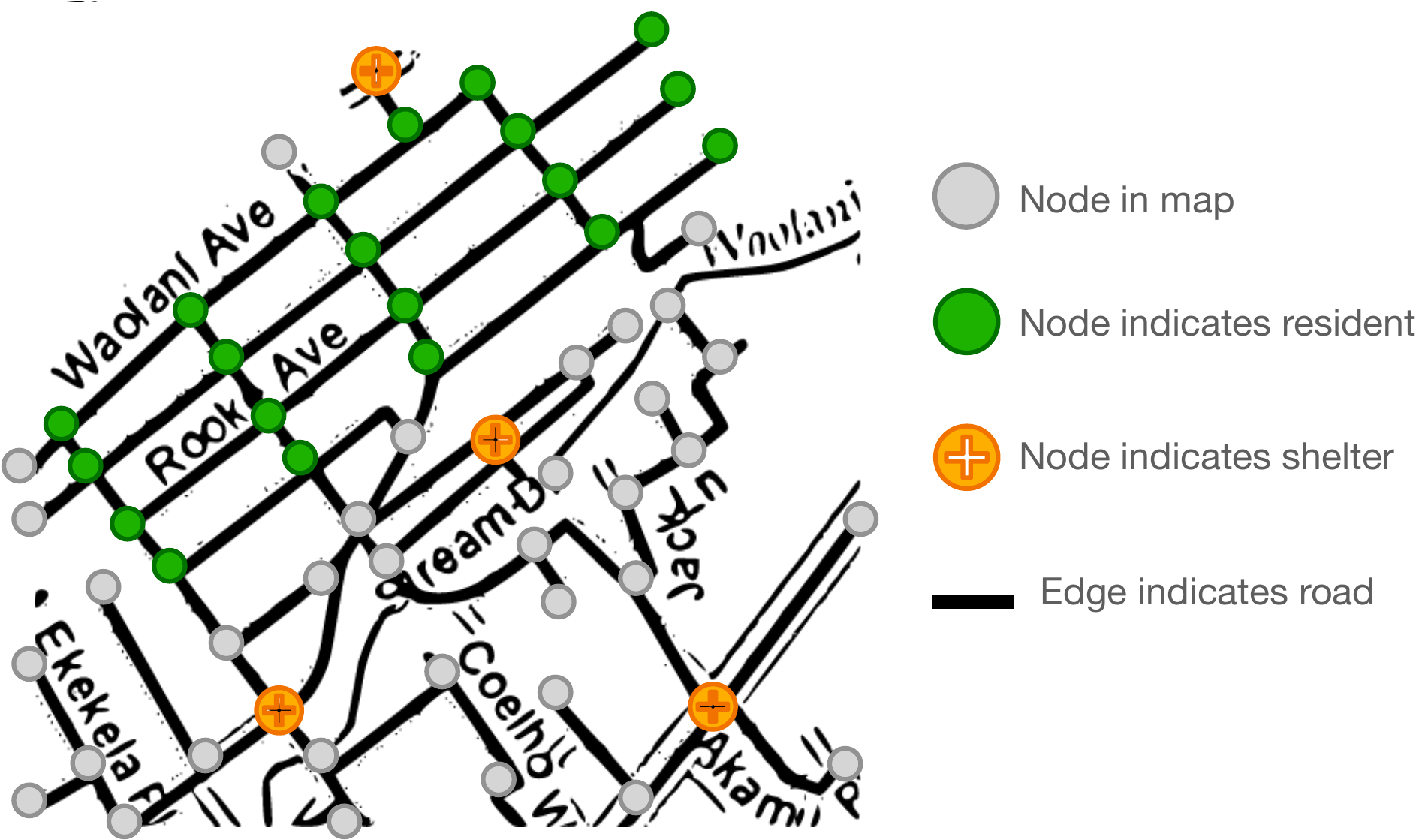}
    \caption{Example assignment of shelters that guarantee sufficient alternative paths from the resident areas, at Hawaii Island. Every orange dot corresponds to shelters and the green dot indicates a resident area.}
    \label{fig:shelter}
\end{figure}

Given a map $G = (V, E)$ where nodes in $V = \{v_1,\dots,v_N\}$ represent $N$ areas and an edge $e = (v_i,v_j) \in E$ indicates a road from $v_i$ to $v_j$, $N$ and $M$ denote the number of nodes and edges, respectively. Given a subset of nodes  $R =\{v_{r_1},\dots, v_{r_k} \} \subseteq V$ indicates the \textit{residential areas},  the task is to choose at most $m$ nodes as shelters from the rest of the nodes, such that the number of routes that can reach a shelter from each residential area is maximized. 
%at least $2^{q_r}$, $q_r \in \dN$, for any residential area $v_{r} \in R$.
Fig.~\ref{fig:shelter} gives an example with  $m=4$ shelters and there are sufficiently many roads connecting the resident area to those shelters.

Current methods~\cite{bayram2018shelter,amideo2019optimising} considered finding shelter locations that have at least one single path from a residential area. However, those proposed methods cannot be generalized to solve the problem that requires sufficient alternative routes from residential area to shelters, primarily because counting the number of paths is intractable. 
This complexity makes it difficult to solve large-scale problems of this type.

\noindent\textbf{SMC Formulation.}  
\embed transforms this optimization problem into a decision problem by gradually increasing the path count threshold $q_r$.
The decision problem decides if there are at least $2^{q_r}$ paths connecting any residential area with a shelter. 
The assigned shelters is represented by a vector $\mathbb{b} = (b_1,\ldots, b_n)\in\{0,1\}^n$, where $b_i=1$ implies node $v_i$ is chosen as shelter. Let $\phi(\mathbf{b})=\left(\sum_{i=1}^n b_i\right) \le m$ represent there are at most $m$ shelters. Let $f(v_r, v_s, E')$ be an indicator function that returns one if and only if the selected edges $E'$ form a path from $v_r$ to $v_s$. The whole formula is:
\begin{align*}
% \label{eq:shelter}
    \phi(\mathbf{b}),b_i\Rightarrow& \left(\sum_{v_s \in S, E'\subseteq E}f(v_r, v_s, E') \geq 2^{q_r}\right)
\text{for } 1\le i\le n.
\end{align*}
We leave the details implementation of $f(v_r, v_s, E')$ in the Appendix C.

\subsection{Empirical Experiment Analysis}
\noindent\textbf{Experiment Setting.} We crawl the real-world dataset from the Hawaii Statewide GIS Program website. We extract the real Hawaii map with those major roads and manually label those resident areas on the map. We create problems of different scales by subtracting different sub-regions from the map. 
3 major resident areas are picked as $R$, and set $m=5$.

\begin{table}[!t]
\centering
\caption{\embed takes less empirical running time than baselines to find shelter location assignments over different graphs. Graph size is the number of nodes in the graph.} 
\label{tab:shelter}
\begin{tabular}{c|ccc}
\hline
&\multicolumn{3}{c}{Graph Size } \\
& $N=121$ & $N=183$  & $N=388$\\
\hline
\embed (ours)  &  $\mathbf{0.04} hr$  & $\mathbf{0.11} hr$      & $\mathbf{0.16} hr$  \\
GibbsSampler-LS   & $0.56 hr$   &  $0.66 hr$ & $6.97 hr$ \\
QuickSampler-LS   & $0.31 hr$ & $0.29 hr$   & $0.62 hr$ \\
Unigen-LS & $0.08 hr$ & $0.17 hr$  & $0.42 hr$ \\
% original data
% & $121$ & $183$ & $246$ & $388$\\
% \midrule
% \embed (ours)  &  $0.04 h$  & $0.11 h$    & $1.03 h$  & $0.16 h$  \\
% Gibbs-LS   & $0.56 h$   &  $0.66 h$ & $6.43 h$ & $6.97 h$ \\
% QuickSampler-LS   & $0.31 h$ & $0.29 h$  & $2.84 h$ & $0.62 h$ \\
% Unigen-LS & $0.08 h$ & $0.07 h$ & $1.99 h$ & $0.42 h$ \\
\hline
\end{tabular}
\end{table}

In terms of baselines, we consider the local search algorithm with shelter locations as the state and the number of paths between shelters and resident areas as the heuristic. Due to the intractability of path counting in our formulation, the heuristic is approximated by querying sampling oracles. In particular, we consider 
1) Gibbs sampling-based \cite{geman1984stochastic} Local Search (Gibbs-LS). 2) Uniform SAT sampler-based  \cite{SGM20} Local Search (Unigen-LS). 3) Quick Sampler-based \cite{dutra2018efficient} Local Search (Quick-LS). Each baseline runs 5 times, and the best result is included. Each run is until the approach finds a local minimum. For our \embed, we give a time limit of 12 hours. The algorithm repeatedly runs with increasing $q_r$ until it times out. The time shown in Table~\ref{tab:shelter} and the number of paths in Figure~\ref{fig:shelter} correspond to the cumulative time and the best solutions found before the algorithm times out.
For the evaluation metrics, we consider 1) the number of paths identified in the predicted plan by each algorithm, and 2) the total running time of the process.

\noindent\textbf{Result Analysis.} In terms of running time (in Table~\ref{tab:shelter}), \embed takes less empirical running time than baselines for finding shelter location assignments over different graphs. In Fig.~\ref{fig:shelter_line}, we evaluate the quality of the predicted shelters by counting the number of connecting paths from residents to the shelters.
The path is counted by directly solving the counting predicate by SharpSAT-TD~\cite{korhonen2021integrating} with given shelter locations. The shelter locations selected by our \embed lead to a higher number of paths than those found by the baselines. 
%To conclude, our proposed \embed can identify shelter location assignments that ensure each residential area has a designated number of alternative paths with high probability. %A formal formulation is as follows.

\begin{figure}[!t]
    \centering
    \includegraphics[width=.99\linewidth]{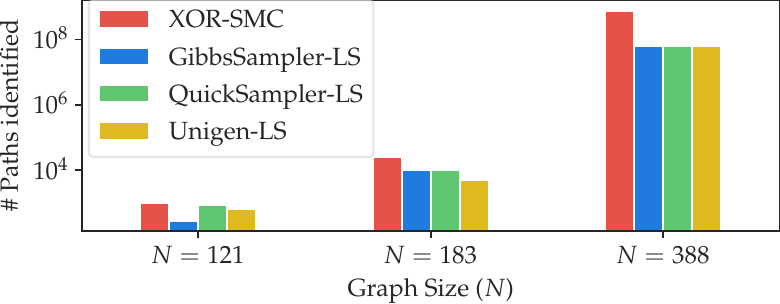}
    \caption{\embed finds better shelter locations with orders of magnitudes more paths from the resident area to the chosen shelters than competing baselines across different graphs (y-axis in log scale).}
    \label{fig:shelter_line}
\end{figure}

\begin{table*}[!t]
    \centering
    \caption{Empirical average of the total production (tons) and running time on two supply networks. Our \embed finds better solutions (higher production) with less time usage compared to baselines.}
    \label{tab:supply}
    \begin{tabular}{r|ccc|ccc}
    \hline
    & \multicolumn{3}{c|}{Small Supply Network} &\multicolumn{3}{c}{Real-world Supply Network}\\
    &\multicolumn{3}{c|}{Disaster Scale ($\%$ of affected edges)} &\multicolumn{3}{c}{Disaster Scale ($\%$ of affected edges)} \\
    & $10\%$ & $20\%$ & $30\%$  & $10\%$ & $20\%$ & $30\%$  \\
    \hline
    \embed (ours)  & $\mathbf{95.0(0.67s)}$  & $\mathbf{96.0(1.19s)} $ & $\mathbf{119.3(17.63s)}$ & $\mathbf{232.2 (57.8s)}$  & $\mathbf{210.2 (75.5s)} $ & $\mathbf{147.8 (96.9s)}$\\
    BP-SAA & $86.5(1.17s)$ & $82.7(3.76s)$ & $118.8(11.48s)$& $187.2 (2 hr)$ & $143.7 (2 hr)$ & $130.5 (2 hr)$\\
    Gibbs-SAA & $88.0(1.24s)$ & $95.4 (4.48s)$ & $113.3 (15.61s)$ &$155.0 (2 hr)$ & $147.1 (2 hr)$ & $111.0 (2 hr)$ \\
    IS-SAA & $80.5 (0.98s)$ & $76.7(3.29s)$ & $109.6(21.74s)$& $157.4 (2 hr)$ & $123.5 (2 hr)$ & $132.9 (2 hr)$\\
    LoopyIS-SAA & $88.0 (0.84s)$ & $95.4 (2.96s)$ & $112.3 (17.95s)$& $200.2 (2 hr)$ & $149.5 (2 hr)$ & $130.5 (2 hr)$\\
    Weighted-SAA & $86.5 (1.11s)$ & $95.4 (3.57s)$ & $117.0(11.80s)$ & $170.5 (2 hr)$ & $169.0 (2 hr)$ & $130.3 (2 hr)$\\
    \hline
    Best possible values & $100.0$ & $104.0$ & $128.0$ & $237.0$ & $226.0$ & $235.0$\\
    \hline
    \end{tabular}
\end{table*}
\section{Experiment 2: Robust Supply Chain Design}

\noindent\textbf{Problem Formulation.} 
Supply chain management found its importance in operations research and economics. The essence of supply chain management is to integrate the flow of products and finances to maximize value to the consumer. Its importance is underscored by the increasing complexity of business environments, where minor inefficiencies by random disasters result in significant extra costs.

Given a supply chain of $N$ suppliers, they form a supply-demand network $(V,E)$ where each node $v\in V$ represents a supplier and edges $e \in E$ represent supply-to-demand trades.
Each supplier $v$ acts as a vendor to downstream suppliers and also as a buyer from upstream suppliers.  We assume a conservation of production conditions for each node, i.e., the input-output ratio is 1.
%
% There are $M$ different kinds of products in the network. 
%
To guarantee substantial production, supplier $v$ should order necessary raw materials from upstream suppliers in advance. Denote the cost of trade between vendor $u$, and buyer $v$ as $c(u,v)$.
Let the amount of goods of the trade between $u$ and $v$ be $f(u,v)$ and $B(v)$ be the total budget to get all his materials ready. 
Due to unpredictable natural disasters and equipment failure, the trade between $u$ and $v$ may fail. Denote $\btheta = (\theta_1, \ldots, \theta_L) \in \{0,1\}^L$ as the state of $L$ different stochastic events, where $\theta_l = 1$ indicates event $l$ occurs.
The objective is to design a global trading plan maximizing the expected total production output, accounting for stochastic influences. This is measured by the output of final-tier suppliers, who produce end goods without supplying others.

Existing works in supply chain optimization often gravitate towards mathematical programming approaches. 
However, when delving into more complex scenarios involving stochastic events—such as uncertain demand or supply disruptions—the task becomes considerably more intricate. Specifically, formulating counting constraints, like guaranteeing a certain amount of supplies across multiple vendors under stochastic events, is intractable. 
These complexities necessitate innovative approaches that can capture the randomness and dynamism inherent in real-world supply chain systems without sacrificing optimality.

%A formal formulation of the problem is as follows:

\noindent\textbf{SMC Formulation.}
Similarly, we transform the maximization into a decision problem by gradually increasing the threshold of the production. Denote the trading plan as $\bx = (x_{u,v}, \dots) \in \{0,1\}^{|E|}$, where $x_{u,v} = 1$ indicates that $v$ purchases raw material from $u$. The decision problem can be formulated into an SMC problem as follows: 
\begin{align} \label{eq:spc_initial}
    &\phi(\bx), \sum_{v \in D} \expect_{\theta} \left[ \sum_{(u,v) \in E_{\bx,\btheta}} f(u,v) \right] \geq 2^q
    % \\
    % &  s.t.~~ \forall v, \sum_{(u,v) \in E_{\bx,\btheta}} c(u,v) \leq B(v) 
    %  \nonumber
\end{align}
where $D$ represents final-tier supplier nodes, $2^q$ is a minimum production level to be guaranteed, $E_{\bx,\btheta}$ is the remaining sets of edges after applying trading plan $\bx$ and under the stochastic events of disasters $\btheta$. $\phi(\bx)$ encodes all essential constraints, e.g., adherence to budget limits, output not exceeding input at each node, etc. 
%$\phi(\bx)=\sum_{(u,v) \in E_{\bx,\btheta}} c(u,v) \leq B(v),\forall v\in V$ ensures that each trade adheres to budget constraints.
 We leave the formulation details in Appendix C.

\subsection{Empirical Experiment Analysis}
\noindent\textbf{Experiment Setting.} 
The dataset is the wheat supply chain network from \cite{zokaee2017robust}. 
The dataset only provides the cost of trade, the capacity of transportation, and raw material demand. We further generate stochastic events of disasters (see Appendix~C) over different portions of supply-demand edges. The random disasters make the expectation computation intractable. For a better comparison of running time, a small-scale synthetic network is included, in which the cost, budget, and capacities are randomly generated. For the small supply network,  the number of nodes in each layer is $[4,4,5,5]$. For the real-world supply network, the number of nodes in each layer is $[9,7,9,19]$. 

For the baseline, we utilize Sample Average Approximation (SAA)-based methods \cite{kleywegt2002sample}. These baselines employ Mixed Integer Programming (MIP) to identify a trading plan that directly maximizes the average production across networks impacted by 100 sampled disasters. The average over samples serves as a proxy for the actual expected production. For the sampler, we consider Gibbs sampling (Gibbs-SAA), belief propagation (BP-SAA), importance sampling (IS-SAA), loopy-importance sampling, and weighted sampling (Weighted-SAA).
For a fair comparison, we imposed a time limit of 30 seconds for the small-sized network and 2 hours for the real-world network. 
The time shown in Table \ref{tab:supply} for SAA approaches is their actual execution time. Our SMC solver again executes repeatedly with increasing $q$ until it times outs. 
The time shown is the cumulative time it finds the best solution (last one) before the time limit.

To evaluate the efficacy of a trading plan, we calculate its empirical average production under $10,000$ i.i.d. disasters, sampled from the ground-truth distribution. The production numbers are reported in Table \ref{tab:supply}. This method is adopted due to the computational infeasibility to calculate expectations directly.
For SAA approaches that exceed the time limit, the production numbers in Table \ref{tab:supply} are for the best solutions found within the time limit.

\noindent\textbf{Result Analysis.} Table~\ref{tab:supply} shows the production and running times of plans derived from various methods. 
% \jzl{[Seems repeated] Specifically, if the accumulated running time of \embed exceeds the limit or fails to find a solution while the threshold approaches $2^q$, the algorithm is terminated and returns the plan found at the previous threshold $2^{q-1}$ with the reported accumulated time being that required to reach $2^{q-1}$.}
For small networks, while SAA-based methods can complete MIP within the time limit, they remain sub-optimal as they optimize a surrogate expectation derived from sampling, which deviates from the true expectation. In the case of larger networks, these methods further struggle due to the poor scalability of the large MIP formulation. They fail to find good solutions within the 2-hour time limit. 
In contrast, \embed, by formulating as an SMC problem, directly optimizes the intractable expectation using XOR counting, and yields superior solution quality and less running time.

\section{Conclusion}
We presented \embed, an algorithm with polynomial approximation guarantees to solve the highly intractable Satisfiability Modulo Counting (SMC) problems.
%
% Locate
Solving SMC problems presents unique challenges due to their intricate nature, integrating statistical inference and symbolic reasoning. 
Prior work on SMC solving offers no or loose guarantees and may find suboptimal solutions. 
%especially in scenarios with combinatorial constraints.
%
\embed~transforms the intractable SMC problem into  satisfiability problems by replacing intricate model counting with SAT formulae subject to  randomized XOR constraints. \embed~also obtains constant approximation guarantees on the solutions obtained. 
SMC solvers offer useful tools for many real-world problems at the nexus of symbolic and statistical AI.
Extensive experiments on two real-world applications in AI for social good demonstrate that \embed~outperforms other approaches in solution quality and running time.

%\newpage
\section*{Acknowledgments}
We thank all the reviewers for their constructive comments.
This research was supported by NSF grant CCF-1918327.

\bibliography{simplified}

\begin{thebibliography}{86}
\providecommand{\natexlab}[1]{#1}

\bibitem[{cpm(2023)}]{cpml2023}
 2023.
\newblock \emph{AAAI-23 Constraint Programming and Machine Learning Bridge Program}.

\bibitem[{nes(2023)}]{nesymas2023}
 2023.
\newblock \emph{Neuro-symbolic AI for Agent and Multi-Agent Systems (NeSyMAS) Workshop at AAMAS-23}.

\bibitem[{Achlioptas and Theodoropoulos(2017)}]{Dachlioptas2017probabilistic}
Achlioptas, D.; and Theodoropoulos, P. 2017.
\newblock Probabilistic Model Counting with Short XORs.
\newblock In \emph{{SAT}}, volume 10491, 3--19. Springer.

\bibitem[{Almeida et~al.(2019)Almeida, Shi, Gomes-Selman, Wu, Xue, Angarita, Barros, Forsberg, Garc{\'i}a-Villacorta, Hamilton, Melack, Montoya, Perez, Sethi, Gomes, and Flecker}]{Almeida2019ReducingGG}
Almeida, R.; Shi, Q.; Gomes-Selman, J.~M.; Wu, X.; Xue, Y.; Angarita, H.; Barros, N.; Forsberg, B.; Garc{\'i}a-Villacorta, R.; Hamilton, S.; Melack, J.; Montoya, M.; Perez, G.; Sethi, S.; Gomes, C.; and Flecker, A. 2019.
\newblock Reducing greenhouse gas emissions of Amazon hydropower with strategic dam planning.
\newblock \emph{Nature Communications}, 10.

\bibitem[{Amideo, Scaparra, and Kotiadis(2019)}]{amideo2019optimising}
Amideo, A.~E.; Scaparra, M.~P.; and Kotiadis, K. 2019.
\newblock Optimising shelter location and evacuation routing operations: The critical issues.
\newblock \emph{Eur. J. Oper. Res.}, 279(2): 279--295.

\bibitem[{Balog et~al.(2017)Balog, Tripuraneni, Ghahramani, and Weller}]{BalogTGW17}
Balog, M.; Tripuraneni, N.; Ghahramani, Z.; and Weller, A. 2017.
\newblock Lost Relatives of the Gumbel Trick.
\newblock In \emph{{ICML}}, volume~70, 371--379. {PMLR}.

\bibitem[{Barbosa et~al.(2022)Barbosa, Barrett, Brain, Kremer, Lachnitt, Mann, Mohamed, Mohamed, Niemetz, N{\"{o}}tzli, Ozdemir, Preiner, Reynolds, Sheng, Tinelli, and Zohar}]{DBLP:conf/tacas/BarbosaBBKLMMMN22}
Barbosa, H.; Barrett, C.~W.; Brain, M.; Kremer, G.; Lachnitt, H.; Mann, M.; Mohamed, A.; Mohamed, M.; Niemetz, A.; N{\"{o}}tzli, A.; Ozdemir, A.; Preiner, M.; Reynolds, A.; Sheng, Y.; Tinelli, C.; and Zohar, Y. 2022.
\newblock cvc5: {A} Versatile and Industrial-Strength {SMT} Solver.
\newblock In \emph{{TACAS}}, volume 13243, 415--442. Springer.

\bibitem[{Barrett et~al.(2009)Barrett, Sebastiani, Seshia, and Tinelli}]{BSST09}
Barrett, C.; Sebastiani, R.; Seshia, S.; and Tinelli, C. 2009.
\newblock \emph{Satisfiability Modulo Theories}, chapter~26, 825--885.

\bibitem[{Barrett et~al.(2021)Barrett, Sebastiani, Seshia, and Tinelli}]{DBLP:series/faia/BarrettSST21}
Barrett, C.~W.; Sebastiani, R.; Seshia, S.~A.; and Tinelli, C. 2021.
\newblock Satisfiability Modulo Theories.
\newblock In \emph{Handbook of Satisfiability}, volume 336, 1267--1329. {IOS} Press.

\bibitem[{Bayram and Yaman(2018)}]{bayram2018shelter}
Bayram, V.; and Yaman, H. 2018.
\newblock Shelter Location and Evacuation Route Assignment Under Uncertainty: {A} Benders Decomposition Approach.
\newblock \emph{Transp. Sci.}, 52(2): 416--436.

\bibitem[{Belanger and McCallum(2016)}]{Dbelanger2016structured}
Belanger, D.; and McCallum, A. 2016.
\newblock Structured Prediction Energy Networks.
\newblock In \emph{{ICML}}, volume~48, 983--992.

\bibitem[{Biere et~al.(2009)Biere, Heule, van Maaren, and Walsh}]{HandbookOfSAT2009}
Biere, A.; Heule, M. J.~H.; van Maaren, H.; and Walsh, T., eds. 2009.
\newblock \emph{Handbook of Satisfiability}, volume 185 of \emph{Frontiers in Artificial Intelligence and Applications}.

\bibitem[{Bj{\o}rner et~al.(2018)Bj{\o}rner, de~Moura, Nachmanson, and Wintersteiger}]{DBLP:conf/setss/BjornerMNW18}
Bj{\o}rner, N.~S.; de~Moura, L.; Nachmanson, L.; and Wintersteiger, C.~M. 2018.
\newblock Programming {Z3}.
\newblock In \emph{{SETSS}}, volume 11430, 148--201. Springer.

\bibitem[{Braunstein, M{\'e}zard, and Zecchina(2005)}]{Braunstein2005SurveyPA}
Braunstein, A.; M{\'e}zard, M.; and Zecchina, R. 2005.
\newblock Survey propagation: an algorithm for satisfiability.
\newblock \emph{Random Struct. Algorithms}, 27: 201--226.

\bibitem[{Chakraborty et~al.(2014)Chakraborty, Fremont, Meel, Seshia, and Vardi}]{Chakraborty2014DistributionAwareSA}
Chakraborty, S.; Fremont, D.~J.; Meel, K.~S.; Seshia, S.~A.; and Vardi, M.~Y. 2014.
\newblock Distribution-Aware Sampling and Weighted Model Counting for {SAT}.
\newblock In \emph{{AAAI}}, 1722--1730.

\bibitem[{Chakraborty, Meel, and Vardi(2013)}]{Chakraborty2013scalable}
Chakraborty, S.; Meel, K.~S.; and Vardi, M.~Y. 2013.
\newblock A Scalable and Nearly Uniform Generator of {SAT} Witnesses.
\newblock In \emph{{CAV}}, volume 8044, 608--623. Springer.

\bibitem[{Chavira and Darwiche(2008)}]{chavira2008probabilistic}
Chavira, M.; and Darwiche, A. 2008.
\newblock On probabilistic inference by weighted model counting.
\newblock \emph{Artificial Intelligence}, 172(6-7): 772--799.

\bibitem[{Cheng et~al.(2012)Cheng, Chen, Dong, Xu, and Ihler}]{cheng2012ApproxSumOp}
Cheng, Q.; Chen, F.; Dong, J.; Xu, W.; and Ihler, A. 2012.
\newblock Approximating the Sum Operation for Marginal-MAP Inference.
\newblock In \emph{{AAAI}}, 1882--1887.

\bibitem[{Choi, Kisa, and Darwiche(2013)}]{ChoiKisaDarwiche13}
Choi, A.; Kisa, D.; and Darwiche, A. 2013.
\newblock Compiling Probabilistic Graphical Models Using Sentential Decision Diagrams.
\newblock In \emph{{ECSQARU}}, volume 7958, 121--132. Springer.

\bibitem[{Choi, Friedman, and den Broeck(2022)}]{ChoiAISTATS22}
Choi, Y.; Friedman, T.; and den Broeck, G.~V. 2022.
\newblock Solving Marginal {MAP} Exactly by Probabilistic Circuit Transformations.
\newblock In \emph{{AISTATS}}, volume 151, 10196--10208. {PMLR}.

\bibitem[{Conitzer and Sandholm(2006)}]{ConitzerS06BayesianStackelberg}
Conitzer, V.; and Sandholm, T. 2006.
\newblock Computing the optimal strategy to commit to.
\newblock In \emph{{EC}}, 82--90. {ACM}.

\bibitem[{Conrad et~al.(2012)Conrad, Gomes, van Hoeve, Sabharwal, and Suter}]{Conrad2012}
Conrad, J.; Gomes, C.~P.; van Hoeve, W.-J.; Sabharwal, A.; and Suter, J.~F. 2012.
\newblock Wildlife corridors as a connected subgraph problem.
\newblock \emph{Journal of Environmental Economics and Management}, 63(1).

\bibitem[{Darwiche and Marquis(2002)}]{Darwiche2002compilation}
Darwiche, A.; and Marquis, P. 2002.
\newblock A Knowledge Compilation Map.
\newblock \emph{J. Artif. Int. Res.}

\bibitem[{de~Moura and Bj{\o}rner(2007)}]{DBLP:conf/cade/MouraB07}
de~Moura, L.~M.; and Bj{\o}rner, N.~S. 2007.
\newblock Efficient E-Matching for {SMT} Solvers.
\newblock In \emph{{CADE}}, volume 4603 of \emph{Lecture Notes in Computer Science}, 183--198. Springer.

\bibitem[{de~Moura and Bj{\o}rner(2008)}]{DBLP:conf/tacas/MouraB08}
de~Moura, L.~M.; and Bj{\o}rner, N.~S. 2008.
\newblock {Z3:} An Efficient {SMT} Solver.
\newblock In \emph{{TACAS}}, volume 4963, 337--340. Springer.

\bibitem[{Dilkina and Gomes(2010)}]{Dilkina2010}
Dilkina, B.; and Gomes, C.~P. 2010.
\newblock Solving Connected Subgraph Problems in Wildlife Conservation.
\newblock In \emph{CPAIOR}, 102--116.

\bibitem[{Dutra et~al.(2018)Dutra, Laeufer, Bachrach, and Sen}]{dutra2018efficient}
Dutra, R.; Laeufer, K.; Bachrach, J.; and Sen, K. 2018.
\newblock Efficient sampling of {SAT} solutions for testing.
\newblock In \emph{{ICSE}}, 549--559. {ACM}.

\bibitem[{Dyer and Stougie(2006)}]{Dyer2006StochasticProgramming}
Dyer, M.~E.; and Stougie, L. 2006.
\newblock Computational complexity of stochastic programming problems.
\newblock \emph{Math. Program.}, 106(3): 423--432.

\bibitem[{Ermon et~al.(2013{\natexlab{a}})Ermon, Gomes, Sabharwal, and Selman}]{ermon2013embed}
Ermon, S.; Gomes, C.~P.; Sabharwal, A.; and Selman, B. 2013{\natexlab{a}}.
\newblock Embed and Project: Discrete Sampling with Universal Hashing.
\newblock In \emph{{NIPS}}, 2085--2093.

\bibitem[{Ermon et~al.(2013{\natexlab{b}})Ermon, Gomes, Sabharwal, and Selman}]{Ermon13Wish}
Ermon, S.; Gomes, C.~P.; Sabharwal, A.; and Selman, B. 2013{\natexlab{b}}.
\newblock Taming the Curse of Dimensionality: Discrete Integration by Hashing and Optimization.
\newblock In \emph{{ICML}}, volume~28, 334--342.

\bibitem[{Fang, Stone, and Tambe(2015)}]{Fang15GreenSecurityGame}
Fang, F.; Stone, P.; and Tambe, M. 2015.
\newblock When Security Games Go Green: Designing Defender Strategies to Prevent Poaching and Illegal Fishing.
\newblock In \emph{{IJCAI}}.

\bibitem[{Feser et~al.(2020)Feser, Madden, Tang, and Solar{-}Lezama}]{DBLP:journals/pacmpl/FeserM0S20}
Feser, J.~K.; Madden, S.; Tang, N.; and Solar{-}Lezama, A. 2020.
\newblock Deductive optimization of relational data storage.
\newblock \emph{Proc. {ACM} Program. Lang.}, 4({OOPSLA}): 170:1--170:30.

\bibitem[{Flerova et~al.(2011)Flerova, Ihler, Dechter, and Otten}]{Flerova2011MinibucketMomentMatching}
Flerova, N.; Ihler, E.; Dechter, R.; and Otten, L. 2011.
\newblock Mini-bucket elimination with moment matching.
\newblock In \emph{NIPS Workshop DISCML}.

\bibitem[{Fredrikson and Jha(2014)}]{fredrikson2014satisfiability}
Fredrikson, M.; and Jha, S. 2014.
\newblock Satisfiability modulo counting: a new approach for analyzing privacy properties.
\newblock In \emph{{CSL-LICS}}, 42:1--42:10. {ACM}.

\bibitem[{Geman and Geman(1984)}]{geman1984stochastic}
Geman, S.; and Geman, D. 1984.
\newblock Stochastic Relaxation, Gibbs Distributions, and the Bayesian Restoration of Images.
\newblock \emph{{IEEE} Trans. Pattern Anal. Mach. Intell.}, 6(6): 721--741.

\bibitem[{Gogate and Dechter(2007)}]{DBLP:journals/jmlr/GogateD07}
Gogate, V.; and Dechter, R. 2007.
\newblock SampleSearch: {A} Scheme that Searches for Consistent Samples.
\newblock In \emph{{AISTATS}}, volume~2, 147--154.

\bibitem[{Gogate and Dechter(2012)}]{GOGATE2012AndOr}
Gogate, V.; and Dechter, R. 2012.
\newblock Importance sampling-based estimation over {AND/OR} search spaces for graphical models.
\newblock \emph{Artif. Intell.}, 184-185: 38--77.

\bibitem[{Gomes, Sabharwal, and Selman(2006{\natexlab{a}})}]{Gomes06XORCounting}
Gomes, C.~P.; Sabharwal, A.; and Selman, B. 2006{\natexlab{a}}.
\newblock Model Counting: {A} New Strategy for Obtaining Good Bounds.
\newblock In \emph{{AAAI}}, 54--61.

\bibitem[{Gomes, Sabharwal, and Selman(2006{\natexlab{b}})}]{Gomes2006Sampling}
Gomes, C.~P.; Sabharwal, A.; and Selman, B. 2006{\natexlab{b}}.
\newblock Near-Uniform Sampling of Combinatorial Spaces Using {XOR} Constraints.
\newblock In \emph{{NIPS}}, 481--488. {MIT} Press.

\bibitem[{Hazan and Jaakkola(2012)}]{HazanJ12Perturbation}
Hazan, T.; and Jaakkola, T.~S. 2012.
\newblock On the Partition Function and Random Maximum A-Posteriori Perturbations.
\newblock In \emph{{ICML}}.

\bibitem[{Hazan and Shashua(2010)}]{Hazan2010NormProductBP}
Hazan, T.; and Shashua, A. 2010.
\newblock Norm-Product Belief Propagation: Primal-Dual Message-Passing for Approximate Inference.
\newblock \emph{{IEEE} Trans. Inf. Theory}, 56(12): 6294--6316.

\bibitem[{Israeli and Wood(2002)}]{israeli2002shortest}
Israeli, E.; and Wood, R.~K. 2002.
\newblock Shortest-path network interdiction.
\newblock \emph{Networks: An International Journal}, 40(2): 97--111.

\bibitem[{Jang, Gu, and Poole(2017)}]{DBLP:conf/iclr/JangGP17}
Jang, E.; Gu, S.; and Poole, B. 2017.
\newblock Categorical Reparameterization with Gumbel-Softmax.
\newblock In \emph{{ICLR} (Poster)}.

\bibitem[{Jerrum and Sinclair(1997)}]{Jerrum:1996:MCM:241938.241950}
Jerrum, M.; and Sinclair, A. 1997.
\newblock \emph{The Markov chain Monte Carlo method: an approach to approximate counting and integration}, 482--520.
\newblock Boston, MA, USA.

\bibitem[{Jerrum, Valiant, and Vazirani(1986)}]{jerrum1986random}
Jerrum, M.; Valiant, L.~G.; and Vazirani, V.~V. 1986.
\newblock Random Generation of Combinatorial Structures from a Uniform Distribution.
\newblock \emph{Theor. Comput. Sci.}, 43: 169--188.

\bibitem[{Jiang, Rai, and III(2011)}]{Jiang11MessagePassingMMAP}
Jiang, J.; Rai, P.; and III, H.~D. 2011.
\newblock Message-Passing for Approximate {MAP} Inference with Latent Variables.
\newblock In \emph{{NIPS}}, 1197--1205.

\bibitem[{K., Shoham, and Gurfinkel(2022)}]{DBLP:journals/pacmpl/KSG22}
K., H. G.~V.; Shoham, S.; and Gurfinkel, A. 2022.
\newblock Solving constrained Horn clauses modulo algebraic data types and recursive functions.
\newblock \emph{Proc. {ACM} Program. Lang.}, 6: 1--29.

\bibitem[{Kempe, Kleinberg, and Tardos(2005)}]{kempe2005influential}
Kempe, D.; Kleinberg, J.; and Tardos, {\'E}. 2005.
\newblock Influential nodes in a diffusion model for social networks.
\newblock In \emph{Automata, languages and programming}, 1127--1138. Springer.

\bibitem[{Kisa et~al.(2014)Kisa, den Broeck, Choi, and Darwiche}]{KisaVCD14}
Kisa, D.; den Broeck, G.~V.; Choi, A.; and Darwiche, A. 2014.
\newblock Probabilistic Sentential Decision Diagrams.
\newblock In \emph{{KR}}.

\bibitem[{Kiselev and Poupart(2014)}]{kiselev2014policy}
Kiselev, I.; and Poupart, P. 2014.
\newblock Policy optimization by marginal-map probabilistic inference in generative models.
\newblock In \emph{{AAMAS}}, 1611--1612. {IFAAMAS/ACM}.

\bibitem[{Kleywegt, Shapiro, and Homem{-}de{-}Mello(2002)}]{kleywegt2002sample}
Kleywegt, A.~J.; Shapiro, A.; and Homem{-}de{-}Mello, T. 2002.
\newblock The Sample Average Approximation Method for Stochastic Discrete Optimization.
\newblock \emph{{SIAM} J. Optim.}, 12(2): 479--502.

\bibitem[{Korhonen and J{\"a}rvisalo(2021)}]{korhonen2021integrating}
Korhonen, T.; and J{\"a}rvisalo, M. 2021.
\newblock Integrating tree decompositions into decision heuristics of propositional model counters (short paper).
\newblock In \emph{27th International Conference on Principles and Practice of Constraint Programming (CP 2021)}. Schloss Dagstuhl-Leibniz-Zentrum f{\"u}r Informatik.

\bibitem[{Kuck et~al.(2019)Kuck, Dao, Zhao, Bartan, Sabharwal, and Ermon}]{kuck2019adaptive}
Kuck, J.; Dao, T.; Zhao, S.; Bartan, B.; Sabharwal, A.; and Ermon, S. 2019.
\newblock Adaptive Hashing for Model Counting.
\newblock In \emph{{UAI}}, volume 115, 271--280. {AUAI} Press.

\bibitem[{Lee et~al.(2016)Lee, Marinescu, Dechter, and Ihler}]{Lee2016AOBB}
Lee, J.; Marinescu, R.; Dechter, R.; and Ihler, A. 2016.
\newblock From Exact to Anytime Solutions for Marginal {MAP}.
\newblock In \emph{{AAAI}}, 3255--3262.

\bibitem[{Liu and Ihler(2012)}]{liu12bBPDecision}
Liu, Q.; and Ihler, A. 2012.
\newblock Belief Propagation for Structured Decision Making.
\newblock In \emph{{UAI}}, 523--532. {AUAI} Press.

\bibitem[{Liu and Ihler(2013)}]{LiuI13VariationalMMAP}
Liu, Q.; and Ihler, A.~T. 2013.
\newblock Variational algorithms for marginal {MAP}.
\newblock \emph{Journal of Machine Learning Research}, 14.

\bibitem[{Madras(2002)}]{madras2002lectures}
Madras, N. 2002.
\newblock \emph{{Lectures on Monte Carlo Methods}}.
\newblock American Mathematical Society.

\bibitem[{Maneva, Mossel, and Wainwright(2007)}]{ManevaMW07}
Maneva, E.~N.; Mossel, E.; and Wainwright, M.~J. 2007.
\newblock A new look at survey propagation and its generalizations.
\newblock \emph{J. {ACM}}, 54(4): 17.

\bibitem[{Marinescu, Dechter, and Ihler(2014)}]{Radu14AndOrMMAP}
Marinescu, R.; Dechter, R.; and Ihler, A.~T. 2014.
\newblock {AND/OR} Search for Marginal {MAP}.
\newblock In \emph{{UAI}}.

\bibitem[{Mau{\'{a}} and de~Campos(2012)}]{Maua2012AnytimeMAP}
Mau{\'{a}}, D.~D.; and de~Campos, C.~P. 2012.
\newblock Anytime Marginal {MAP} Inference.
\newblock In \emph{{ICML}}.

\bibitem[{Mei, Jiang, and Tu(2018)}]{mei2018maximum}
Mei, J.; Jiang, Y.; and Tu, K. 2018.
\newblock Maximum {A} Posteriori Inference in Sum-Product Networks.
\newblock In \emph{{AAAI}}, 1923--1930.

\bibitem[{Munawar et~al.(2023)Munawar, Lenchner, Rossi, Horesh, Gray, and Campbell}]{neurosymbolic2023}
Munawar, A.; Lenchner, J.; Rossi, F.; Horesh, L.; Gray, A.; and Campbell, M., eds. 2023.
\newblock \emph{IBM Neuro-Symbolic AI Workshop 2023 -- Unifying Statistical and Symbolic AI}.

\bibitem[{Papandreou and Yuille(2010)}]{Yuille2010GaussianSampling}
Papandreou, G.; and Yuille, A.~L. 2010.
\newblock Gaussian sampling by local perturbations.
\newblock In \emph{{NIPS}}, 1858--1866.

\bibitem[{Papandreou and Yuille(2011)}]{DBLP:conf/iccv/PapandreouY11}
Papandreou, G.; and Yuille, A.~L. 2011.
\newblock Perturb-and-MAP random fields: Using discrete optimization to learn and sample from energy models.
\newblock In \emph{{ICCV}}, 193--200.

\bibitem[{Park and Darwiche(2004)}]{Park04MMAPComplexity}
Park, J.~D.; and Darwiche, A. 2004.
\newblock Complexity Results and Approximation Strategies for MAP Explanations.
\newblock \emph{J. Artif. Int. Res.}

\bibitem[{Ping, Liu, and Ihler(2015)}]{Wei15DecompositionBoundMMAP}
Ping, W.; Liu, Q.; and Ihler, A. 2015.
\newblock Decomposition Bounds for Marginal {MAP}.
\newblock In \emph{{NIPS}}, 3267--3275.

\bibitem[{Rossi, van Beek, and Walsh(2006)}]{rossi06}
Rossi, F.; van Beek, P.; and Walsh, T., eds. 2006.
\newblock \emph{Handbook of Constraint Programming}, volume~2.
\newblock Elsevier.

\bibitem[{Shapiro(2003)}]{shapiro2003monte}
Shapiro, A. 2003.
\newblock Monte Carlo sampling methods.
\newblock \emph{Handbooks in operations research and management science}, 10: 353--425.

\bibitem[{Sheldon et~al.(2010)Sheldon, Dilkina, Elmachtoub, Finseth, Sabharwal, Conrad, Gomes, Shmoys, Allen, Amundsen, and Vaughan}]{Sheldon10NetworkDesign}
Sheldon, D.; Dilkina, B.; Elmachtoub, A.~N.; Finseth, R.; Sabharwal, A.; Conrad, J.; Gomes, C.~P.; Shmoys, D.~B.; Allen, W.; Amundsen, O.; and Vaughan, W. 2010.
\newblock Maximizing the Spread of Cascades Using Network Design.
\newblock In \emph{{UAI}}, 517--526.

\bibitem[{Sontag et~al.(2008)Sontag, Meltzer, Globerson, Jaakkola, and Weiss}]{Sontag08tightenLP}
Sontag, D.; Meltzer, T.; Globerson, A.; Jaakkola, T.; and Weiss, Y. 2008.
\newblock Tightening LP Relaxations for MAP using Message Passing.
\newblock In \emph{UAI}, 503--510.

\bibitem[{Soos, Gocht, and Meel(2020)}]{SGM20}
Soos, M.; Gocht, S.; and Meel, K.~S. 2020.
\newblock Tinted, Detached, and Lazy {CNF-XOR} Solving and Its Applications to Counting and Sampling.
\newblock In \emph{{CAV}}, volume 12224 of \emph{Lecture Notes in Computer Science}, 463--484. Springer.

\bibitem[{Swamy and Shmoys(2006)}]{swamy2006stochasticOptimization}
Swamy, C.; and Shmoys, D.~B. 2006.
\newblock Approximation Algorithms for 2-Stage Stochastic Optimization Problems.
\newblock In \emph{{FSTTCS}}, volume 4337, 5--19. Springer.

\bibitem[{Valiant and Vazirani(1986)}]{Valiant1986uniqueSAT}
Valiant, L.~G.; and Vazirani, V.~V. 1986.
\newblock {NP} is as Easy as Detecting Unique Solutions.
\newblock \emph{Theor. Comput. Sci.}, 47(3): 85--93.

\bibitem[{Verweij et~al.(2003)Verweij, Ahmed, Kleywegt, Nemhauser, and Shapiro}]{Verweij2003SAA}
Verweij, B.; Ahmed, S.; Kleywegt, A.~J.; Nemhauser, G.~L.; and Shapiro, A. 2003.
\newblock The Sample Average Approximation Method Applied to Stochastic Routing Problems: {A} Computational Study.
\newblock \emph{Comput. Optim. Appl.}, 24(2-3): 289--333.

\bibitem[{Wainwright, Jaakkola, and Willsky(2003)}]{WainwrightJW03}
Wainwright, M.~J.; Jaakkola, T.~S.; and Willsky, A.~S. 2003.
\newblock Tree-reweighted belief propagation algorithms and approximate {ML} estimation by pseudo-moment matching.
\newblock In \emph{{AISTATS}}. Society for Artificial Intelligence and Statistics.

\bibitem[{Wainwright and Jordan(2008)}]{Wainwright2008}
Wainwright, M.~J.; and Jordan, M.~I. 2008.
\newblock Graphical Models, Exponential Families, and Variational Inference.
\newblock \emph{Found. Trends Mach. Learn.}, 1(1-2): 1--305.

\bibitem[{Welling and Teh(2003)}]{welling2003approximate}
Welling, M.; and Teh, Y.~W. 2003.
\newblock Approximate inference in Boltzmann machines.
\newblock \emph{Artif. Intell.}, 143(1): 19--50.

\bibitem[{Williams and Snyder(2005)}]{Williams2005}
Williams, J.~C.; and Snyder, S.~A. 2005.
\newblock Restoring Habitat Corridors in Fragmented Landscapes using Optimization and Percolation Models.
\newblock \emph{Environmental Modeling and Assessment}, 10(3): 239--250.

\bibitem[{Wu et~al.(2017)Wu, Kumar, Sheldon, and Zilberstein}]{wu2017robust}
Wu, X.; Kumar, A.; Sheldon, D.; and Zilberstein, S. 2017.
\newblock Robust Optimization for Tree-Structured Stochastic Network Design.
\newblock In \emph{{AAAI}}, 4545--4551.

\bibitem[{Wu, Sheldon, and Zilberstein(2014)}]{wu2014stochastic}
Wu, X.; Sheldon, D.; and Zilberstein, S. 2014.
\newblock Stochastic Network Design in Bidirected Trees.
\newblock In \emph{{NIPS}}, 882--890.

\bibitem[{Wu, Sheldon, and Zilberstein(2015)}]{Wu2015}
Wu, X.; Sheldon, D.; and Zilberstein, S. 2015.
\newblock Fast Combinatorial Algorithm for Optimizing the Spread of Cascades.
\newblock In \emph{{IJCAI}}, 2655--2661.

\bibitem[{Xue, Fern, and Sheldon(2015)}]{Xue15ScheduleCascade}
Xue, S.; Fern, A.; and Sheldon, D. 2015.
\newblock Scheduling Conservation Designs for Maximum Flexibility via Network Cascade Optimization.
\newblock \emph{J. Artif. Intell. Res.}, 52: 331--360.

\bibitem[{Xue, Choi, and Darwiche(2012)}]{XueCD12}
Xue, Y.; Choi, A.; and Darwiche, A. 2012.
\newblock Basing Decisions on Sentences in Decision Diagrams.
\newblock In \emph{{AAAI}}, 842--849.

\bibitem[{Xue et~al.(2016)Xue, Li, Ermon, Gomes, and Selman}]{Xue2016MarginalMAP}
Xue, Y.; Li, Z.; Ermon, S.; Gomes, C.~P.; and Selman, B. 2016.
\newblock Solving Marginal {MAP} Problems with {NP} Oracles and Parity Constraints.
\newblock In \emph{{NIPS}}, 1127--1135.

\bibitem[{Yedidia, Freeman, and Weiss(2000)}]{yedidia2001generalized}
Yedidia, J.~S.; Freeman, W.~T.; and Weiss, Y. 2000.
\newblock Generalized Belief Propagation.
\newblock In \emph{{NIPS}}, 689--695. {MIT} Press.

\bibitem[{Zokaee et~al.(2017)Zokaee, Jabbarzadeh, Fahimnia, and Sadjadi}]{zokaee2017robust}
Zokaee, S.; Jabbarzadeh, A.; Fahimnia, B.; and Sadjadi, S.~J. 2017.
\newblock Robust supply chain network design: an optimization model with real world application.
\newblock \emph{Annals of Operations Research}, 257: 15--44.

\end{thebibliography}

\appendix
\clearpage
\onecolumn
\section{Proofs of Theorem~\ref{th:xor-smc}}
\label{apx:proof_claim2}

\subsection{Proof of Theorem~\ref{th:xor-smc} Claim 2}
Suppose for all $\bx \in \{0,1\}^n$ and $\bb \in \{0,1\}^k$, there is 
\begin{align*}
    \neg \left( \phi(\bx,\bb) \wedge \left( \bigwedge_{i=1}^k \left(b_i \Rightarrow \sum_{y_i} f_i(\bx, y_i) \geq 2^{q_i - c} \right) \right) \right) \\
    \Rightarrow 
    \neg \phi(\bx,\bb) \vee \left( \bigvee_{i=1}^k \left( b_i \wedge \sum_{y_i} f_i(\bx, y_i) \leq 2^{q_i -c} \right) \right).
\end{align*}
Consider a fixed $\bx_1$ and $\bb_1$, we will examine the probability of not returning true when discovering $(\bx_1, \bb_1)$.

If $\neg \phi(\bx_1,\bb_1)$ holds true, then the probability of \embed($\phi$, $\{f_i\}_{i=1}^k$, $\{q_i\}_{i=1}^k$, $T$) not returning true for discovering $(\bx_1, \bb_1)$ is one. 

Otherwise, when $\bigvee_{i=1}^k \left( b_i \wedge \sum_{y_i} f_i(\bx_1, y_i) \leq 2^{q_i -c} \right)$ is true, i.e, 
\begin{align} \label{eq:condition2}
    \exists i, ~s.t.~~b_i \wedge \sum_{y_i} f_i(\bx_1, y_i) \leq 2^{q_i -c}.
\end{align}
Let’s select any $t = 1, \dots, T$. The Boolean formula $\psi_t$ can be simplified by substituting the values of $\bx_1$ and $\bb_1$:
\begin{align*}
    \psi_t = &\left( f_1(\bx_1, y_1^{(t)}) \wedge \mathtt{XOR}_1(y_1^{(t)}) \wedge \dots \wedge \mathtt{XOR}_{q_1}(y_1^{(t)}) \right) \nonumber \\
    &\wedge \dots \wedge \nonumber \\
    & \left( f_{k_1}(\bx_1, y_{k_1}^{(t)}) \wedge \mathtt{XOR}_1(y_{k_1}^{(t)}) \wedge \dots \wedge \mathtt{XOR}_{q_{k_1}}(y_{k_1}^{(t)}) \right) \nonumber
\end{align*}
where $k_1$ is the number of digits that are non-zero in $\bb_1$, and $1 \leq k_1 \leq k$.

According to Eq.~\eqref{eq:condition2}, we define 
$$
\gamma_i = \left( f_{i}(\bx_1, y_{i}^{(t)}) \wedge \mathtt{XOR}_1(y_{i}^{(t)}) \wedge \dots \wedge \mathtt{XOR}_{q_{i}}(y_{i}^{(t)}) \right)
$$
where $\sum_{y_{i}^{(t)}} f_i(\bx_1, y_{i}^{(t)}) \leq 2^{q_i -c}$. Then according to Lemma~\ref{th:wish}, with probability at least $1 - \frac{2^c}{(2^c - 1)^2}$, $(\bx_1, y_{i}^{(t)})$ renders $\gamma_i$ false for all $y_{i}^{(t)}$.

Then the probability of $\psi_t$ being false under all $(\bx_1, \bb_1, y_{1}^{(t)},\dots, y_{k}^{(t)})$ is
\begin{align*}
    & \prob((\bx_1, \bb_1, y_{1}^{(t)},\dots, y_{k}^{(t)}) \text{ renders $\psi_t$ false for all $\{y_{i}^{(t)} \}_{i=1}^k$}) \\
    = & \prob\left( \bigvee_{i=1}^{k_1} (\bx_1, y_i^{(t)}) \text{ renders $\gamma_i$ false for all $y_{i}^{(t)}$} \right) \\
    \geq & \prob((\bx_1, y_i^{(t)}) \text{ renders $\gamma_i$ false for all $y_{i}^{(t)}$}) \\
    \geq & 1 - \frac{2^c}{(2^c - 1)^2}
\end{align*}

Define $\Gamma_t$ as a binary indicator variable where
\begin{align*}
    \Gamma_t = \begin{cases}
        0 & \text{if all $(\bx_1, \bb_1, y_{1}^{(t)},\dots, y_{k}^{(t)})$ renders $\psi_t$ false} \\
        1 & \text{otherwise}
    \end{cases}
\end{align*}

Upon discovering $\bx_1,\bb_1$, \embed won't return true if the majority of $\psi_t$, $t = 1,\dots,T$ are false under $(\bx_1, \bb_1, y_{1}^{(t)},\dots, y_{k}^{(t)})$ for all $\{y_{i}^{(t)} \}_{i=1}^k$, that is, $\sum_{t=1}^T \Gamma_{t} \leq \frac{T}{2}$. By Chernoff–Hoeffding theorem, 
\begin{align*}
    \prob\left( \sum_{t=1}^T \Gamma_t \leq \frac{T}{2} \right) &= 1 - \prob\left( \sum_{t=1}^T \Gamma_t > \frac{T}{2} \right) \geq 1 - e^{-\alpha(c,1) T} 
\end{align*}

For $T \geq \frac{(n+k)\ln2 - \ln \eta}{\alpha(c,1)}$, it follows that $e^{-\alpha(c,1) T} \leq \frac{\eta}{2^{n+k}}$, In this case, with a probability at least $1-\frac{\eta}{2^{n+k}}$, \embed won't return true upon discovering $\bx_1$ and $\bb_1$.

Noted that \embed returns false only if it won't return true upon discovering any $(\bx_1, \bb_1)$. Therefore, we can write
\begin{align*}
    & \prob(\text{\embed returns false}) \\ 
    =~&\prob(\forall \bx, \bb, \text{\embed won't return true upon $(\bx, \bb)$})\\
    =~&1 - \prob(\exists \bx_1, \bb_1, \text{\embed returns true upon $(\bx_1, \bb_1)$}) \\
    \geq ~& 1 - \sum_{\bx_1,\bb_1} \prob( \text{\embed returns true upon $(\bx_1, \bb_1)$}) \\
    \geq ~& 1 - \frac{\eta}{2^{n+k}} 2^{n+k} = 1 - \eta.
\end{align*}

\subsection{Proof of Theorem 2}

Combining both Claim 1 and 2, we can choose the value $T = \frac{(n+k)\ln2 - \ln \eta}{\alpha(c,k)}$ such that $T \geq \max\{\frac{- \ln \eta}{\alpha(c,k)}, \frac{(n+k)\ln2 - \ln \eta}{\alpha(c,1)} \}$. This completes the proof. 

\section{Implementation of \embed} \label{apx:implement}

% Please find our online code repository at:
% \begin{mdframed}
% \centering
% \url{https://anonymous.4open.science/r/xor_smc-2EF3}
% \end{mdframed}
Our code is attached to the supplementary material.
It contains  1) the implementation of our \embed method 2) the list of datasets, and 3) the implementation of several baseline algorithms.

We use CPLEX solver\footnote{\url{https://www.ibm.com/products/ilog-cplex-optimization-studio/cplex-optimizer}} version 22.1.1 as the NP oracle. We use Python package boolexpr\footnote{\url{http://www.boolexpr.org/index.html}} to process those parts involve logic operations ``$\lor, \land,\neg$''.

% \subsection{Discretization}
% Generally, for any weighted function $w(x):\{0,1\}^n \rightarrow \dR^+$, define the following embedding

% \begin{align*}
%     S(w, l) = \left\{ (x,y) | \forall 1 \leq i \leq l, \frac{w(x)}{M} \leq \frac{2^{i-1}}{2^l} \Rightarrow y_i=0 \right\}
% \end{align*}
% where $M=\max_x w(x)$.

% \begin{lemma}
% \label{lm:discretize}
%     \cite{Xue2016MarginalMAP} Let $w_l'(x,y)$ be an indicator variable which is 1 if and only if $(x,y)$ is in $S(w,l)$, i.e., $w_l'(x,y) = \mathbf{1}_{(x,y) \in S(w,l)}$, we claim that,
%     \begin{align*}
%         \sum_x w(x) \leq \frac{M}{2^l} \sum_{(x,y)} w_l'(x,y) \leq 2 \sum_x w(x) + M 2^{n-l}
%     \end{align*}
% \end{lemma}

\subsection{Computational Pipeline of \embed}

\subsubsection{Preprocess}
The preprocessing step converts the problem into a standard form as specified in Eq.~\eqref{eq:overall}: 
\begin{align*}
    \phi(\mathbf{x}, \mathbf{b}) &\wedge
     \left[b_1 \Rightarrow \left(\sum_{\mathbf{y}_1\in\mathcal{Y}_{1}} f_1(\mathbf{x}, \mathbf{y}_1) \geq 2^{q_1}\right) \right] \wedge \dots \wedge \left[b_k \Rightarrow \left(\sum_{\mathbf{y}_k\in\mathcal{Y}_{k}} f_k(\mathbf{x}, \mathbf{y}_k) \geq 2^{q_k}\right) \right].
\end{align*}
The crucial aspects are to guarantee
\begin{itemize}
    \item $\forall i$, $\mathcal{Y}_i$ should be properly encoded such that $\mathcal{Y}_i \subseteq \{0,1\}^{d_i}$ for some $d_i$. (See example in Shelter Location Assignment.)
    \item $\forall i$, $f_i:\pX \times \pY_i\rightarrow \{0,1\}$ must be a binary-output function. (See example in Robust Supply Chain Design.)
\end{itemize}

\subsubsection{Formulate a binary-output function into a Boolean Formula}

A binary-output function $f_i$ can be formulated into a Boolean formula where $\bx$ and $\by_i$ are seen as boolean variables and $f_i(\bx,\by_i)=1$ iff. the boolean formula is satisfiable at $(\bx, \by_i)$. 

Although it seems straightforward to use SAT solvers, we found it more efficient to transform the SAT into an equivalent Constraint Programming (CP) problem and query the CPLEX solver after comparison. Due to the interchangeability of CP and SAT, we will keep using the SAT formulation for the following discussion.

\subsubsection{Add XOR Constraints}
Then we need to add random XOR constraints to the Boolean formula. As an example, an XOR constraint for $\by \in \{0,1\}^d$ is 
\begin{align*}
    \mathtt{XOR}(\by) = y_2 \oplus y_3 \oplus y_d \oplus 1
\end{align*}
In this $\mathtt{XOR}(\by)$, each literal $y_j$, $j=1,\dots,d$, and the extra $1$ in the end (adding an extra $1$ can be seen as negating all literals) has $1/2$ chance of appearance. Following this manner, we will sample multiple XOR constraints for each $\by_i$. The number of XOR constraints for $\by_i$ is determined by the logarithm of the corresponding threshold $\log_2(2^{q_i}) = q_i$.

After sampling all XOR constraints, we will have a boolean formula:
\begin{align}
\label{eq:bool_w_xor}
    & (b_1 \Rightarrow f_1(\mathbf{x}, \mathbf{y}_1)\wedge \mathtt{XOR}_1(\mathbf{y}_1) \wedge \ldots \wedge \mathtt{XOR}_{q_1}(\mathbf{y}_1) ) \wedge \nonumber\\
    & (b_2 \Rightarrow f_2(\mathbf{x}, \mathbf{y}_2)\wedge \mathtt{XOR}_2(\mathbf{y}_2) \wedge \ldots \wedge \mathtt{XOR}_{q_2}(\mathbf{y}_2) ) \wedge \\
    & \cdots \nonumber
\end{align}

\subsubsection{Repeat and Examine the Majority}
Since all XOR constraints are sampled uniformly, the chance of Eq.~\eqref{eq:bool_w_xor} being satisfiable when Eq.~\eqref{eq:overall} holds true may not be significant. Therefore we can enhance the probability by repeating experiments $T$ times and checking whether the majority is satisfiable.

\section{Experiment Settings}

All experiments run on a high-end server with two Sky Lake CPUs @ 2.60GHz, 24 Core, 96 GB RAM.

\subsection{Shelter Location Assignment} \label{apx:shelter}
 
As discussed in the main text, the essential decision problem decides there are at least $2^{q_r}$ paths connecting any residential area with a shelter. 
The assigned shelters is represented by a vector $\bb = (b_1,\ldots, b_n)\in\{0,1\}^n$, where $b_i=1$ implies node $v_i$ is chosen as shelter. Let $\phi(\mathbf{b})=\left(\sum_{i=1}^n b_i\right) \le m$ represent there are at most $m$ shelters. Let $I(v_r, v_s, E')$ be an indicator function that returns one if and only if the selected edges $E'$ form a path from $v_r$ to $v_s$. The whole formula is:
\begin{align*}
    \phi(\mathbf{b}) \wedge b_i\Rightarrow& \left(\sum_{v_s \in S, E'\subseteq E} I(v_r, v_s, E') \geq 2^{q_r}\right)
\text{for } 1\le i\le n.
\end{align*}

\subsubsection{Path Encoding} 
The path indicator function $I(\cdot)$ is implemented as follows. 
Any path between 2 nodes can be mapped to a unique flow where the flow along the path is 1 and others are 0. One such unit integer flow can be represented by a binary-valued vector:
$$
\bdf \in \{(f_{e_1},f_{e_2},\dots,f_{e_M}) | f_{e_i} \in \{0,1\},\text{ for } 1\le i\le M\},
$$
where $f_{e_i} = 1$ if and only if the amount of flow on edge $e_i$ is 1. Then we can encode every path on the map using $\bdf$.
Then we can use a flow indicator function instead in the implementation, where
\begin{align*}
    I(v_r, v_s, \bdf) = \begin{cases}
        1 & \text{ vector } \bdf \text{ represents a path from $v_r$ to shelter $v_s$}, \\
        0 & \text{otherwise}.
    \end{cases}
\end{align*}
Then $I(\cdot)$ can be implemented by encoding flow constraints.

\subsubsection{Shelter Location Encoding}

The set of nodes with shelters $S$ can be further encoded by a binary vector
$$
\bx_S  = (x_1,\dots,x_N) \in \{0,1\}^{|V|},
$$
where $x_i = 1$ represents a shelter assigned to node $v_i$, i.e., $v_i \in S$. For example, 
\begin{equation}
    \bx_S = (1,1,1,\underbrace{0,\dots,0}_{\text{all $0$}})
\end{equation}
implies we place only 3 shelters on node $v_1$, $v_2$, and $v_3$.

After we properly define the notations of this problem, the shelter location problem can be formulated into an SMC:
\begin{align}
\label{eq:shelter-encoded}
    &\left( \sum_{i=1}^{|V|} x_i \leq m \right) \wedge     \bigwedge_{v_r \in R} \left(\sum_{v_s \in S}\sum_{\bdf} I( v_r,v_s, \bdf) \geq 2^{q_r}\right)
\end{align}

If problem \eqref{eq:shelter-encoded} holds true for some $S$, then the shelter assignment plan meets the safety requirement. Thus 
in \embed, we are actually solving an SAT problem:
\begin{align*}
% \label{eq:shelter-xor}
    \left( \sum_{i=1}^{|V|} x_i \leq m \right) \wedge  \bigwedge_{v_r} \left( \underbrace{ I( v_r, v_s, \bdf) \wedge \mathtt{XOR}_1(\bdf, \bx_S) \wedge \dots \wedge \mathtt{XOR}_{q_r}(\bdf, \bx_S)}_{\text{repeat $T$ times and check the satisfiability of the majority}} \right)
\end{align*}
where we determine whether there is such a $\bdf$ and $\bx_S$ renders the boolean formula above true.

% \subsubsection{Optimization Objective}
% To find the best shelter location, we apply binary search to find the maximum $q_r$ makes Eq. \eqref{eq:shelter-encoded} true. The corresponding shelter location assignment is considered optimized.

\subsubsection{Dataset} The dataset is from Shelters at Hawaii\footnote{\url{https://iu.maps.arcgis.com/apps/Styler/index.html?appid=477fae419a014a3ba7c95cdcb39fd946}}.  The case study for this assignment is from Hawaii and involves Mauna Loa volcano.  This map offers an analysis of the most practical locations for emergency shelters for displaced people within 15 minutes of their original locations.  Relationships between population, emergency shelters, and high-risk lava flow zones on the islands are analyzed. According to different choices of sub-regions, we extract 3 maps with 121, 186, and 388 nodes representing 3 difficulty levels with 237, 359, and 775 undirected edges, respectively. Those maps are included in the code repository.

\subsubsection{Evaluation Metric}
Suppose we find a shelter assignment $\{v_{s_1},\dots, v_{s_q}\}$ for residential areas $\{v_{r_1},\dots, v_{r_p}\}$. The quality of the assignment is evaluated by the total paths from the resident areas to the chosen shelters, which is computed as:
$$
\sum_{i = 1}^{p} \sum_{j = 1}^{q} \text{Number of paths from $v_{r_i}$ to $v_{s_j}$}
$$
However, the path counting problem is an intractable problem even if shelter locations are given. We use SharpSAT-TD \cite{korhonen2021integrating} (a model counting solver) to solve each counting predicate separately with given shelter locations.

\subsubsection{Choice of Baseline}

For baseline methods, shelter locations are found by local search with the shelter assignment as the search state and a path approximation estimation function as 
\begin{align*}
& h(\{v_{s_1},\dots, v_{s_q}\}) = \sum_{i = 1}^{p}  \sum_{j = 1}^{q} \text{Approximated number of paths from $v_{r_i}$ to $v_{s_j}$}    
\end{align*}
The approximated path counting is achieved by querying approximate sampling oracles including Gibbs Sampler, Unigen Sampler\footnote{\url{ https://github.com/meelgroup/unigen}}, Quick sampler\footnote{\url{https://github.com/RafaelTupynamba/quicksampler }}, etc.

\subsubsection{Hyper Parameters}
For all experiments, residential nodes are picked at node $v_{r_1} = 0$, $v_{r_2} = 10$, $v_{r_3} =20$. The maximum number of shelters is $m=5$. Parameter $T$ (for repeat experiments) is set to be 2, which empirically works well.

Due to the sensitivity of local search to the search state initialization. We initialize the starting point randomly and repeat 5 times. Each run is until the approach finds a local minimum and only those best are picked for comparison. For our \embed, we give a time limit of 12 hours. The algorithm repeatedly runs with increasing $q_r$ until it times out. The time shown in Table~\ref{tab:shelter} and the number of paths in Figure~\ref{fig:shelter} correspond to the cumulative time and the best solutions found before the algorithm times out.

\subsection{Robust Supply Chain Design} \label{apx:supply-chain}

\begin{figure}
    \centering
    \includegraphics[width=\linewidth]{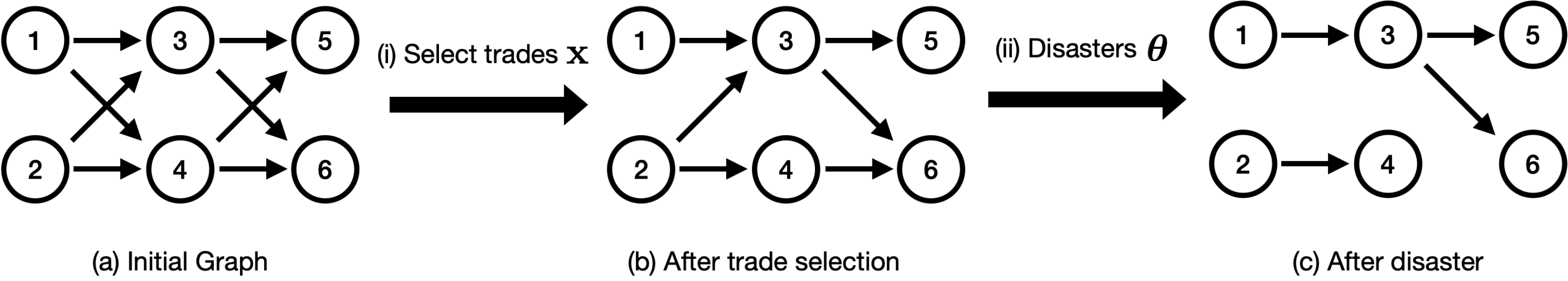}
    \caption{Example of the supply chain design application. (a) The initial supply chain where nodes 1 and 2 are suppliers of wheat; nodes 3 and 4 are producers of bread; nodes 5 and 6 are markets. (b) Graph after trade (edge) selection of each demander. (c) Under the effect of random events, e.g., natural disasters, some trades (edges) are disabled.}
    \label{fig:spc_example}
\end{figure}

Figure~\ref{fig:spc_example} shows an example of the supply chain design. 
This supply chain (network) can be formulated into a directed acyclic graph $G(V,E)$. For each node $v\in V$ in the supply chain,
(1) it represents a supplier, which consumes raw materials from upper-stream suppliers (nodes in the previous layer) and supplies products to down-stream suppliers (nodes in the following layer). The number of nodes is $|V|$. (2) Each node $v$ has a limited budget, denoted by $Budget(v)$, representing the maximum budget it can spend on purchasing raw materials. (3)  We call those nodes without upstream suppliers in the initial graph "\textbf{primary suppliers}", e.g., nodes 1 and 2 in Fig.~\ref{fig:spc_example}. Those without downstream suppliers are called "\textbf{final-tier suppliers}", e.g., nodes 5 and 6 in Fig.~\ref{fig:spc_example}.

For each edge $(u,v)\in E$ in the graph, (1) it represents the possible trades from nodes $u$ to $v$, and the direction is from a supplier to a demand. The number of edges is denoted as $|E|$. (2) Each edge $(u,v)$ has a capacity $Capacity(u,v)$ representing the maximum amount of raw materials that can flow from the supplier to the demand. (3) Each edge $(u,v)$ has a cost for the demander, denoted by $Cost(u,v)$. It represents the cost of the trade (mainly from purchase and transportation). (4) The actual amount of goods flow on each edge could be below the capacity, we use $flow(e)$ to denote the actual flow on edge $e$.

Due to the limited budget, each node $v$ can only trade with a limited number of suppliers. Denote this selection as $\bx \subseteq E$ where $(u,v)\in \bx$ if and only if the trading edge $(u,v)$ is selected. %Denote the graph after trade selection is $G(V, {\bx})$. 
The trade selection is shown in Figure~\ref{fig:spc_example} (i), and the resulting graph is Figure~\ref{fig:spc_example} (b). We have to ensure $\sum_{ (u,v)\in {\bx}}Capacity(u,v)  \leq Budget(v)$ for all $v$.

% For simplicity, we assume the disaster is associated with edges. We define $Pr_\theta(u,v)$ to take value 0 if the edge $(u,v)$ is 100\% sure to have the disaster and take value 1 if the edge $(u,v)$ will never meet any disaster. When $0<Pr_\theta(u,v)<1$, it implies the edge $(u,v)$ has probability value $Pr_\theta(u,v)$ to survive after the randomized disaster event. In our formulation, we consider the same disaster to be applied over multiple edges. In the real world, it could be the wheat factory and the bread market are in the same city and facing the same natural disaster.  A natural disaster may destroy a group of trades together. Thus the distribution $Pr_\theta$ takes several edges as input argument. The learning algorithm has access to the distribution of the disaster, but can only predict the set of edges to be included in the trade planning.   An example is in Figure~\ref{fig:spc_example} (ii) and (c). In this work, $\prob(\btheta)$ is effectively modeled using a Bayesian Network.

A disaster is defined as: $\btheta \in \{(\theta_{u,v},\dots)|(u,v)\in E\} \subseteq \{0,1\}^{|E|}$ where $\theta_{u,v} = 1$ iff. the disaster destroies edge $(u,v)$. The disaster $\btheta$ is associated with a random distribution $\prob(\btheta)$. The random disaster happens after the trade selection. An example is in Figure~\ref{fig:spc_example} (ii) and (c). 
In our study, we address the complex scenario where a disaster might simultaneously affect multiple trades, especially those located within the same geographical region, e.g., a wheat factory and a bread market in the same city are likely to face the same natural disaster. 
Bayesian Networks are exceptionally expressive and capable of effectively representing conditional dependencies and probabilistic relationships among various entities. This makes them particularly well-suited for capturing the nuanced interdependencies among trades.
By modeling $\prob(\btheta)$ as a Bayesian Network, we can more accurately reflect the real-world scenario thus providing a more realistic case study analysis.

\subsubsection{Problem Formulation} As discussed in the main text Eq. \eqref{eq:spc_initial}, the core decision problem to determine if there is a trading plan $\bx$ that satisfies
\begin{align*} 
    &\phi(\bx) \wedge \sum_{v \text{ in final-tier suppliers }} \expect_{\btheta} \left[ \sum_{(u,v) \text{ in selected edges } \bx} flow(u,v;\btheta) \right] \geq 2^q.
\end{align*}
where the summation inside the expectation quantifies all incoming raw materials transported to $v$ on the actual trading map following trade selection $\bx$ and random disaster $\btheta$, and the total expectation term represents the expected production of node $v$ subject to disaster distribution. $\phi(\bx)$ captures the budget limitation $\sum_{ (u,v) \text{ in selected edges } {\bx}} Cost(u,v) \leq Budget(v)$ for all node $v$.

Note that if the flow on each edge $e$ can vary arbitrarily between 0 and $Capacity(e)$, then a representation of $flow(\cdot, \cdot; \btheta)$ must be included among the decision variables alongside $\bx$. This is because the same graph can have varying amounts and assignments of flow. As a result, the decision problem becomes overly complex, as encoding $flow(\cdot,\cdot; \btheta)$ for all possible graphs requires exponential space. 

To simplify, we assume that all suppliers, except for the primary and final-tier suppliers, are capable of meeting the maximum production requirement with a single incoming trade, that is, $Capacity(u,v) \geq \sum_{(v,w)\in E} Capacity(v,w), \forall (u,v) \in E$. Consequently, the amount of flow on selected trading edges must be full to ensure the maximum flow. Therefore, we can get rid of $flow$ and consider $Capacity$ directly.

\subsubsection{SMC Formulation} After simplification, we can expand the decision problem in Eq. \eqref{eq:spc_initial} as 
\begin{align*} 
    & \sum_{\btheta} \prob(\btheta) \sum_{\substack{v \text{ in final-tier suppliers,} \\\text{and }  (u,v) \text{ in selected edges } {\bx}}} Capacity(u,v) I(u,v;\btheta)  \geq 2^q 
    \\
    & \text{such that }~~ \forall v, \sum_{ (u,v) \text{ in selected edges } {\bx}} Cost(u,v) \leq Budget(v)
     \nonumber
\end{align*}
where $I (u,v;\btheta)$ is an indicator function, which is defined as follows:
\begin{align*}
    I(u,v;\btheta) = \begin{cases}
    1 & \text{If there exists a path from one primary supplier to node $v$ via edge $(u,v)$ subject to disasters}\\
    0 & \text{otherwise}.
    \end{cases}
\end{align*}

\subsubsection{Implementation of \embed} To utilize \embed, we need to transform $\prob(\btheta) \sum_{\substack{v \text{ in final-tier suppliers,} \\\text{and }  (u,v) \text{ in selected edges } {\bx}}}Capacity(u,v) I(u,v;\btheta)$ into a sum over 0/1 output function (Boolean formula).

\begin{itemize}
    \item \textbf{Discretize Probability}:
    First, let's discretize $\prob(\btheta)$. Suppose we use 4 digits $\bb = (b_0,b_1,b_2,b_3) \in \{0,1\}^4$ to represent probability from $\frac{0}{15}$ to $\frac{15}{15}$. Then we have,
    \begin{align}
    \label{eq:dis_prob}
        & \prob(\btheta) = \sum_{\bb \in \{0,1\}^4} \underbrace{\mathbf{1}\left (\frac{1}{15}(2^0 b_0 + 2^1 b_1 + 2^2 b_2 + 2^3 b_3) < \prob(\btheta) \right) }_{\text{This indicator function is $0/1$ output, and easy to be implemented in CPLEX}}
    \end{align}

    We model $\prob(\btheta)$ as a Bayesian Network, i.e., $\prob(\btheta) = \prob(\btheta') \cdot \prob(\btheta''|\btheta')$ where $\btheta'$ and $\btheta''$ are two subset of random variables of $\btheta$. 
    \begin{align}
    \label{eq:dis_prob_bayes}
        \prob(\btheta') \cdot \prob(\btheta''|\btheta') =\sum_{\bb' \in \{0,1\}^4} \sum_{\bb'' \in \{0,1\}^4} &\mathbf{1}\left (\frac{1}{15}(2^0 b_0' + 2^1 b_1' + 2^2 b_2' + 2^3 b_3') < \prob(\btheta') \right) \times \\
        & \mathbf{1}\left (\frac{1}{15}(2^0 b_0'' + 2^1 b_1'' + 2^2 b_2'' + 2^3 b_3'') < \prob(\btheta''|\btheta') \right) \nonumber
    \end{align}

    In addition, we need to code up each factor in the Bayesian network, which has a form of $\prob(\btheta')$ or $\prob(\btheta''|\btheta')$ in CPLEX. Use $\prob(\btheta''|\btheta')$ as an example, we need a CPLEX \textit{IloNumExpr}-typed numerical expression that takes the value assignment of $\btheta' \cup \btheta''$ as input and outputs the value of $\prob(\btheta''|\btheta)$. 

    For a clear demonstration, assume we have a factor $P(\theta_1 | \theta_2)$ as in Tab.~\ref{tab:spc_factor}. Then we can encode it in CPLEX as:

    \begin{table}[!t]
    \centering
    \caption{Factor $P(\theta_1 | \theta_2)$} 
    \label{tab:spc_factor}
    \begin{tabular}{|c|c|c|}
    \hline
    $\theta_1$  &  $\theta_2$  &   $\prob(\theta_1| \theta_2)$   \\
    \hline
    0   & 0  & $p_{0}$ \\
    \hline
    0   & 1  & $p_{1}$ \\
    \hline
    1   & 0  & $p_{2}$ \\
    \hline
    1   & 1  & $p_{3}$ \\
    \hline
    \end{tabular}
    \end{table}
    
    \begin{align*}
        P(\theta_1 | \theta_2) &= p_0 u_{0,0} + p_1 u_{0,1} + p_2 u_{1,0} + p_3 u_{1,1}, \\
        s.t.,\quad  & u_{0,0} + u_{0,1} = 1 - \theta_1 \\
        & u_{1,0} + u_{1,1} = \theta_1\\
        & u_{0,0} + u_{1,0} = 1 - \theta_2 \\
        & u_{0,1} + u_{1,1} = \theta_2 \\
        &u_{0,0}, u_{0,1},u_{1,0},u_{1,1} \in \{0,1\}
    \end{align*}
    where $u_{i,j}$ are assisting variables. Then a corresponding indicator function can be implemented as
    \begin{align*}
        & \mathbf{1}\left (\frac{1}{15}(2^0 b_0 + 2^1 b_1 + 2^2 b_2 + 2^3 b_3) < \prob(\theta_1|\theta_2) \right) \\ 
        \equiv & \left( \frac{1}{15}(2^0 b_0 + 2^1 b_1 + 2^2 b_2 + 2^3 b_3) < p_0 u_{0,0} + p_1 u_{0,1} + p_2 u_{1,0} + p_3 u_{1,1} \right) \wedge \\
        & (u_{0,0} + u_{0,1} = 1 - \theta_1) \wedge (u_{1,0} + u_{1,1} = \theta_1) \wedge (u_{0,0} + u_{1,0} = 1 - \theta_2) \wedge (u_{0,1} + u_{1,1} = \theta_2)
    \end{align*}

    After discretization, we could now transform the original problem into
    \begin{align*}
        \exists \bx, \quad \sum_{\btheta}\sum_{b_{\prob}} \text{Indicator function for } \prob(\btheta) \sum_{\substack{v \text{ in final-tier suppliers,} \\\text{and }  (u,v) \text{ in selected edges } {\bx}}}Capacity(u,v) I(u,v;\btheta)
    \end{align*}
     where $b_{\prob}$ represents the digits for representing $\prob(\btheta)$, i.e., $b$ in Eq.~\eqref{eq:dis_prob} and $b' \cup b''$ in Eq.~\eqref{eq:dis_prob_bayes}.
    
    \item \textbf{Discretize Capacity}: The same idea applies to capacity as well. Now we will discretize $ \sum_{\substack{v \text{ in final-tier suppliers,} \\\text{and }  (u,v) \text{ in selected edges } {\bx}}} Capacity(u,v) I(u,v;\btheta)$ as a whole. Assume this total capacity can be represented by a 6-digit binary number $\bb=(b_0,b_1,b_2,b_3,b_4,b_5)$.
    \begin{align*}
        &\sum_{\substack{v \text{ in final-tier suppliers,} \\\text{and }  (u,v) \text{ in selected edges } {\bx}}} Capacity(u,v) I(u,v;\btheta) \\
        =& \sum_{\bb \in \{0,1\}^{6}} \mathbf{1} \left( b_0 + 2b_1+...+2^5b_5 < \sum_{\substack{v \text{ in final-tier suppliers,} \\\text{and }  (u,v) \text{ in selected edges } {\bx}}} Capacity(u,v) I(u,v;\btheta) \right)
    \end{align*}

    \item \textbf{Practical example of $I(u,v;\btheta)$}: This indicator function indicates whether edge $(u,v)$ is in some path from one primary supplier to $v$. 

    Use Figure~\ref{fig:spc_example}(a) as an example. Set $v = v_5$ and examine $I(v_3,v_5)$.  
    \begin{align*}
        I(v_3,v_5; \btheta) = \underbrace{\left( I(v_1,v_3;\btheta) \vee I(v_2,v_3;\btheta) \right)}_{\substack{\text{Node $v_3$ is connected to some primary supplier} }} \wedge \underbrace{x_{v_3,v_5} \wedge \neg \theta_{v_3,v_5}}_{\text{$(v_3,v_5)$ is selected and not destroied}}
    \end{align*}
    
    \item \textbf{Overall}: After discretization and transformation, the problem is a typical \#SAT problem.
    \begin{align*}
        \exists \bx, \quad \sum_{\btheta}\sum_{\bb_{\prob}} \sum_{\bb_c} \text{A $0/1$ output function}(\bx, \btheta, \bb_{\prob}, \bb_c)
    \end{align*}
    where $\bb_{\prob}, \bb_c$ are binary variables introduced by discretizing probability and capacity.
\end{itemize}

\subsubsection{Dataset}
We evaluate our algorithm on a real-world wheat supply chain network from \cite{zokaee2017robust}, the structure of the supply chain is shown in Fig.~\ref{fig:wheat_chain}. For a better comparison, one small-scale synthetic network, in which the parameters are generated with a Gaussian distribution fitted from the real-world data. The statistics of the two supply chain networks are in Table~\ref{tab:dataset-stat}.

\begin{table}[!ht]
    \centering
    \caption{Statistics of the Wheat supply chain network. }
    \label{tab:dataset-stat}
    \begin{tabular}{c|cccc}
    \hline
         &  Wheat suppliers &  Flour factories &  Bread factories & Markets \\
    \hline
    Real-world supply network & 9 & 7  & 9 & 19 \\
        Small synthesized supply network & 4  & 4 & 5 & 5 \\
    \hline
    \end{tabular}
    
\end{table}

In our experiments, we consider different rates of disasters over the supply chain network. We randomly pick 10\%, 20\%, 30\% of edges over the total $M$ edges in the network. In each disaster distribution, The joint probability over affected edges is defined as a randomly generated Bayesian Network.

\subsubsection{Evaluation Metric}

To evaluate the efficacy of a trading plan, we calculate its empirical average of actual production under $K = 10000$ i.i.d. disasters, denoted as $\{\btheta^{(1)}, \dots, \btheta^{(K)}\}$, sampled from the ground-truth distribution. This method is adopted due to the computational infeasibility of directly calculating expectations. Formally, given a trading plan $\bx$, the performance of $\bx$ is evaluated by
\begin{align*}
     &\frac{1}{K}\sum_{k=1}^K \sum_{\substack{v \text{ in final-tier suppliers,} \\\text{and }  (u,v) \text{ in selected edges } {\bx}}} Capacity(u,v) I(u,v;\btheta)
\end{align*}

\subsubsection{Choice of Baseline}
For the baseline, we utilize Sample Average Approximation (SAA)-based methods \cite{kleywegt2002sample}. These baselines employ Mixed Integer Programming (MIP) to identify a trading plan that directly maximizes the average production across networks impacted by 100 sampled disasters. The average over samples serves as a proxy for the actual expected production. For the sampler, we consider Gibbs sampling (Gibbs-SAA), belief propagation (BP-SAA), importance sampling (IS-SAA), loopy-importance sampling, and weighted sampling (Weighted-SAA).
For a fair comparison, we imposed a time limit of 30 seconds for the small-sized network and 2 hours for the real-world network. 
The time shown in Table \ref{tab:supply} for SAA approaches is their actual execution time. Our SMC solver again executes repeatedly with increasing $q$ until it times outs. 
The time shown is the cumulative time it finds the best solution (the last one) before the time limit.

\begin{figure*}[!t]
    \centering
\includegraphics[width=0.75\linewidth]{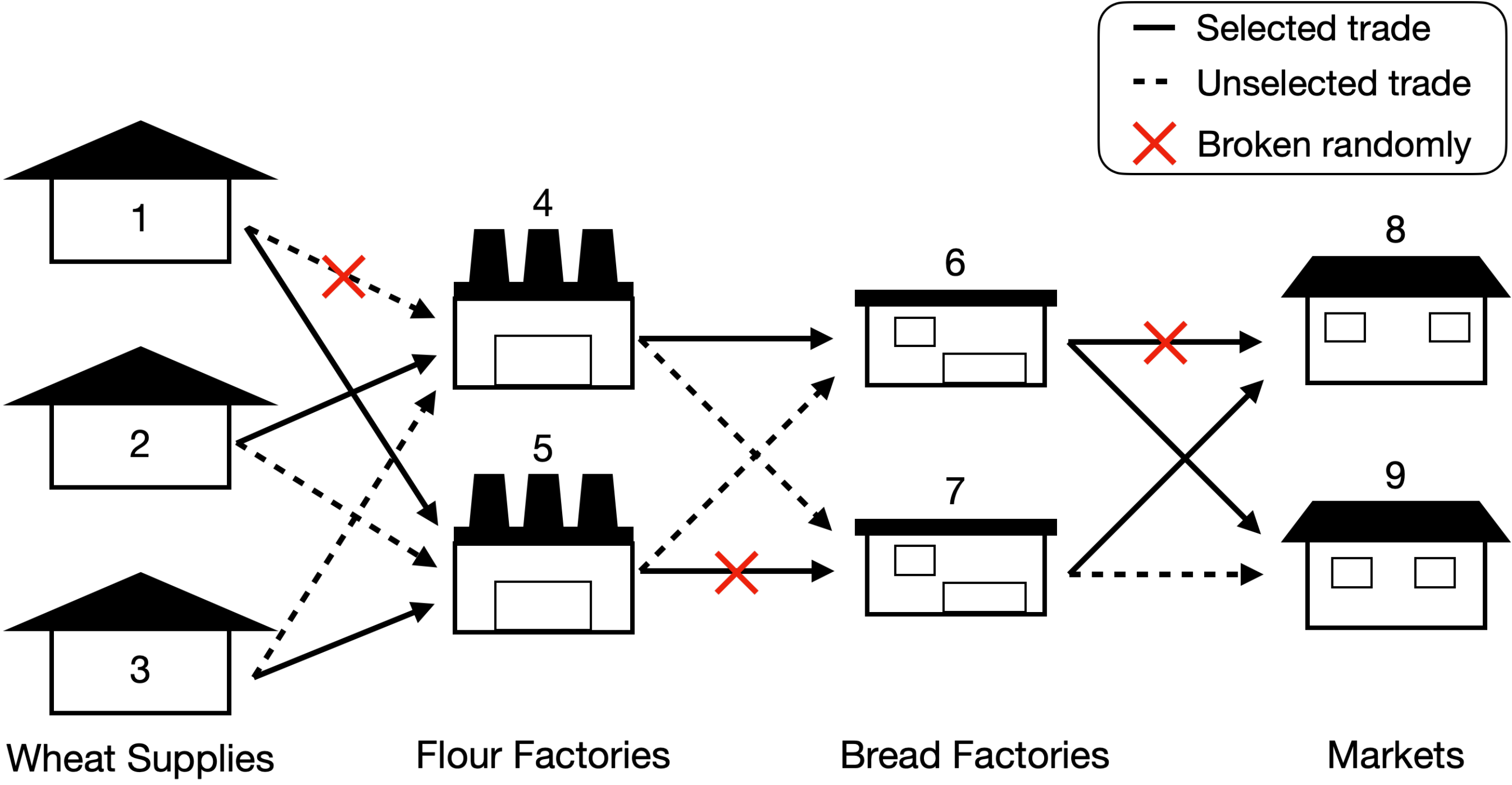}
    \caption{Example of the bread production supply chain network from \cite{zokaee2017robust}. The example includes 3 wheat suppliers, 2 flour factories, 2 bread factories, and 2 markets.
    A black solid arrow from $i$ to $j$ indicates a trade between them ($x_{i,j}=1$ in the trading plan), and a dashed arrow from $p$ to $q$ represents an unchosen trade ($x_{p,q}=0$). Each edge also incorporates a cost and capacity which are not shown in the figure. The red cross indicates the edge is destroyed due to stochastic events (with a probability specified by $\prob(\btheta)$).
    }
\label{fig:wheat_chain}
\end{figure*}

\subsubsection{Hyper Parameters}
Parameter $T$ (for repeat experiments) is set to be 2, which is sufficient to show superior performance over baselines.
$\prob(\btheta)$ is discretized into 4 bits, and capacity is discretized into 12 bits.
In the randomly generated disasters, each Bayesian network node can have at most 8 parents and the number of Bayesian network edges is around half of the maximum possible number. Those disaster distributions as well as the generation script are available in the code repository.

\end{document}